\documentclass[letterpaper]{article} 
\usepackage{aaai23}  
\usepackage{times}  
\usepackage{helvet}  
\usepackage{courier}  
\usepackage[hyphens]{url}  
\usepackage{graphicx} 
\usepackage{natbib}  
\usepackage{caption} 
\usepackage{algorithm}
\usepackage{algorithmic}

\usepackage{newfloat}
\usepackage{listings}
\DeclareCaptionStyle{ruled}{labelfont=normalfont,labelsep=colon,strut=off} 
\floatstyle{ruled}
\newfloat{listing}{tb}{lst}{}
\floatname{listing}{Listing}
\title{Misspecification in Inverse Reinforcement Learning}

\author {
    Joar Skalse, 
    Alessandro Abate 
}
\affiliations {
    Oxford University, Department of Computer Science\\
    joar.skalse@cs.ox.ac.uk, aabate@cs.ox.ac.uk 
}

\newcommand{\States}{\mathcal{S}}
\newcommand{\Actions}{\mathcal{A}}
\newcommand{\reward}{R}
\newcommand{\TransitionDistribution}{\tau}
\newcommand{\InitStateDistribution}{\mu_0} 
\newcommand{\discount}{\gamma}

\newcommand{\MDP}{\langle \States, \Actions, \TransitionDistribution, \InitStateDistribution, \reward, \discount \rangle}

\usepackage{caption, amsmath, amssymb, mathtools, extpfeil, float, yfonts, csquotes, lipsum, fancyhdr, algorithm, algorithmic, amsthm} 

\theoremstyle{plain}
\newtheorem{theorem}{Theorem}[section]

\newtheorem{lemma}[theorem]{Lemma}

\theoremstyle{definition}
\newtheorem{definition}[theorem]{Definition}

\theoremstyle{remark}

\newcommand{\SxA}{{\States{\times}\Actions}}
\newcommand{\SxAxS}{{\States{\times}\Actions{\times}\States}}

\newcommand{\M}{\mathcal{M}}
\newcommand{\Q}{Q}
\newcommand{\V}{V}
\newcommand{\A}{A}
\newcommand{\Return}{G}
\newcommand{\Evaluation}{\mathcal{J}}

\newcommand{\Qfor}[1]{\Q^{#1}}
\newcommand{\Vfor}[1]{\V^{#1}}
\newcommand{\Afor}[1]{\A^{#1}}
\newcommand{\QStar}{\Q^\star}
\newcommand{\VStar}{\V^\star}
\newcommand{\AStar}{\A^\star}

\newcommand{\policy}{\pi}
\newcommand{\OptimalPolicy}{\policy_\star}

\newcommand{\Expect}[2]{\mathbb{E}_{#1}\left[{#2}\right]}

\newcommand{\PS}{\mathrm{PS}_\gamma}
\newcommand{\SR}{S'\mathrm{R}_\TransitionDistribution}
\newcommand{\LS}{\mathrm{LS}}
\newcommand{\OPT}{\mathrm{OP}_{\TransitionDistribution,\gamma}}

\newcommand{\CS}{\mathrm{CS}}

\newcommand{\kPS}[1]{\mathrm{PS}_{\discount,\InitStateDistribution}^{#1}}

\newcommand{\Rspace}{{\hat{\mathcal{R}}}}

\begin{document}

\maketitle

\begin{abstract}
The aim of Inverse Reinforcement Learning (IRL) is to infer a reward function $R$ from a policy $\pi$. 
To do this, we need a model of how $\pi$ relates to $R$. In the current literature, the most common models are \emph{optimality}, \emph{Boltzmann rationality}, and \emph{causal entropy maximisation}. 
One of the primary motivations behind IRL is to infer human preferences from human behaviour.  
However, the true relationship between human preferences and human behaviour is much more complex than any of the models currently used in IRL. This means that they are \emph{misspecified}, which raises the worry that they might lead to unsound inferences if applied to real-world data.
In this paper, we provide a mathematical analysis of how robust different IRL models are to misspecification, and answer precisely how the demonstrator policy may differ from each of the standard models before that model leads to faulty inferences about the reward function $R$. We also introduce a framework for reasoning about misspecification in IRL, together with formal tools that can be used to easily derive the misspecification robustness of new IRL models.
\end{abstract}

\section{Introduction}

Inverse Reinforcement Learning (IRL) is an area of machine learning concerned with inferring what objective an agent is pursuing based on the actions taken by that agent \citep{ng2000}. 
IRL roughly corresponds to the notion of \emph{revealed preferences} in psychology and economics, since it aims to infer \emph{preferences} from \emph{behaviour} \citep{rothkopf2011}. 
IRL has many possible applications. For example, it has been used in scientific contexts, as a tool for understanding animal behaviour \citep{celegansIRL}.
It can also be used in engineering contexts; many important tasks can be represented as sequential decision problems, where the goal is to maximise a \emph{reward function} over several steps \citep{sutton2018}. However, for many complex tasks, it can be very challenging to manually specify a reward function that incentivises the intended behaviour. IRL can then be used to \emph{learn} a good reward function, based on demonstrations of correct behaviour \cite[e.g. ][]{abbeel2010,singh2019}. 
Overall, IRL relates to many fundamental questions about goal-directed behaviour and agent-based modelling.

There are two primary motivations for IRL. The first motivation is to use IRL as a tool for \emph{imitation learning} \cite[e.g.][]{imitation2017}. 
For these applications, it is not fundamentally important whether the learnt reward function actually corresponds to the true intentions of the demonstrator, as long as it helps the imitation learning process. 
The second motivation is to use IRL to understand an agent's preferences and motives \cite[e.g.][]{CIRL}. 
From this perspective, the goal is to learn a reward that captures the demonstrator's true intentions. 
This paper was written with mainly the second motivation in mind.



An IRL algorithm must make assumptions about how the preferences of an agent relate to its behaviour. Most IRL algorithms are based on one of three models; \emph{optimality}, \emph{Boltzmann rationality}, or \emph{causal entropy maximisation}.
These behavioural models are very simple, whereas the true relationship between a person's preferences and their actions of course is incredibly complex. 
In fact, there are observable differences between human data and data synthesised using these standard assumptions \citep{orsini2021}. 
This means that the behavioural models are \emph{misspecified}, which raises the concern that they might systematically lead to flawed inferences if applied to real-world data.


In this paper, we study how robust the behavioural models in IRL are to misspecification. To do this, we first introduce a theoretical framework for analysing misspecification robustness in IRL.
We then derive a number of formal tools for inferring the misspecification robustness of IRL models, and apply these tools to exactly characterise what forms of misspecification the standard IRL models are (or are not) robust to. 
Our analysis is general, as it is carried out in terms of \emph{behavioural models}, rather than \emph{algorithms}, which means that our results will apply to any algorithm based on these models. 
Moreover, the tools we introduce can also be used to easily derive the misspecification robustness of new behavioural models, beyond those we consider in this work.

The motivation behind this work is to provide a theoretically principled understanding of whether and when IRL methods are (or are not) applicable to the problem of inferring a person's (true) preferences and intentions.
Human behaviour is very complex, and while a behavioural model can be more or less accurate, it will never be realistically possible to create a behavioural model that is completely free from misspecification (except possibly for in narrow domains). 
Therefore, if we wish to use IRL as a tool for preference elicitation, then it is crucial to have an understanding of how robust the IRL problem is to misspecification.
In this paper, we contribute towards building this understanding.





\subsection{Related Work}


It is well-known that the standard behavioural models of IRL are misspecified in most applications.
However, there has nonetheless so far not been much research on this topic. 
\citet{choicesetmisspecification} study the effects of \emph{choice set misspecification} in IRL (and reward inference more broadly), following the formalism of \citet{RRIC}. Our work is wider in scope, and aims to provide necessary and sufficient conditions which fully describe the kinds of misspecification to which each behavioural model is robust.
In the field of statistics more broadly, misspecification is a widely studied issue  \cite{white_1994}.


There has been a lot of work on \emph{reducing} misspecification in IRL. One approach to this is to manually add more detail to the models \citep{IgnorantAndInconsistent, IrrationalityCanHelp}, and another approach is to try to \emph{learn} the behavioural model from data \citep{armstrong2019occams, shah2019feasibility}. 
In contrast, our work aims to understand how sensitive IRL is to misspecification (and thus to answer the question of how much misspecification has to be removed).


\citet{skalse2022} study the \emph{partial identifiability} of various reward learning models. Our work uses similar techniques, 
and can be viewed as an extension of their work. The issue of partial identifiability in IRL has also been studied by \citet{ng2000, dvijotham2010, cao2021, kim2021}.

We will discuss the question of what happens if a reward function is changed or misspecified.
This question is also investigated by many previous works, including e.g.\ \citet{epic, rewardhacking, MDPcalculus, reward_misspecification}. 


\subsection{Preliminaries}

A \emph{Markov Decision Processes} (MDP) is a tuple
$(\States, \Actions, \TransitionDistribution, \InitStateDistribution, 
\reward, \discount)$
where
  $\States$ is a set of \emph{states},
  $\Actions$ is a set of \emph{actions},
  $\TransitionDistribution : \SxA \rightsquigarrow \States$ is a \emph{transition function},
  $\InitStateDistribution \in \Delta(\States)$ is an \emph{initial state
  distribution}, 
  $\reward : \SxAxS \to \mathbb{R}$ is a \emph{reward
    function}, 
    and $\discount \in (0, 1]$ is a \emph{discount rate}.
Here $f : X \rightsquigarrow Y$ denotes a probabilistic mapping from $X$ to $Y$.
In this paper, we assume that $\States$ and $\Actions$ are finite.
A \textit{policy} is a function $\policy : \States \rightsquigarrow \Actions$.
A \emph{trajectory} $\xi = \langle s_0, a_0, s_1, a_1 \dots \rangle$ is a possible path in an MDP. The \emph{return function} $\Return$ gives the
cumulative discounted reward of a trajectory,
$\Return(\xi) = \sum_{t=0}^{\infty} \discount^t \reward(s_t, a_t, s_{t+1})$, and the \emph{evaluation function} $\Evaluation$ gives the expected trajectory return given a policy, $\Evaluation(\pi) = \Expect{\xi \sim \pi}{G(\xi)}$. A policy maximising $\Evaluation$ is an \emph{optimal policy}. 
The \emph{value function} $\Vfor{\policy} : \States \rightarrow \mathbb{R}$ of a policy encodes the expected future discounted reward from each state when following that policy. The $Q$-function is $\Qfor{\policy}(s,a) = \Expect{}{R(s,a,S') + \discount \Vfor\policy(S')}$, and the \emph{advantage function} is $\Afor{\policy}(s, a) = \Qfor{\policy}(s, a) - \Vfor{\policy}(s)$.
$\QStar$, $\VStar$, and $\AStar$ denote the $Q$-, value, and advantage
functions of the optimal policies. 
In this paper, we assume that all states in $S$ are reachable under $\TransitionDistribution$ and $\InitStateDistribution$.

In IRL, it is typically assumed that the preferences of the observed agent are described by a reward function $R$, that its environment is described by an MDP, and that its behaviour is described by a (stationary) policy $\pi$. An IRL algorithm also needs a \emph{behavioural model} of how $\pi$ relates to $R$. In the current IRL literature, the most common models are:
\begin{enumerate}
    \item \emph{Optimality}: We assume that $\pi$ is optimal under $R$ (e.g.\ \citet{ng2000}).
    \item \emph{Boltzmann Rationality}: We assume that $\mathbb{P}(\pi(s) = a) \propto e^{\beta Q^\star(s,a)}$, where $\beta$ is a temperature parameter (e.g.\ \citet{ramachandran2007}). 
    %
    \item \emph{Maximal Causal Entropy}: We assume that $\pi$ maximises the causal entropy objective, which is given by $\mathbb{E}[\sum_{t=0}^{\infty} \discount^t (\reward(s_t, a_t, s_{t+1}) + \alpha H(\pi(s_{t+1})))]$, where $\alpha$ is a weight and $H$ is the Shannon entropy function (e.g.\ \citet{ziebart2010thesis}).
\end{enumerate}
In this paper, we will often talk about pairs or sets of reward functions. In these cases, we will give each reward function a subscript $R_i$, and use $\Evaluation_i$, $\VStar_i$, and $\Vfor{\pi}_i$, and so on, to denote $R_i$'s evaluation function, optimal value function, and $\pi$ value function, and so on.





\section{Theoretical Framework}

We here introduce the theoretical framework that we will use to analyse how robust various behavioural models are to misspecification. This framework is rather abstract, but 
it is quite powerful, and
makes our analysis easy to carry out.

\subsection{Definitions and Framework}\label{section:framweork}

For a given set of states $\States$ and set of actions $\Actions$, let $\mathcal{R}$ be the set of all reward functions $R : \SxAxS \rightarrow \mathbb{R}$ definable with $\States$ and $\Actions$. 
Moreover, if $P$ and $Q$ are partitions of a set $X$, we write $P \preceq Q$ if $x_1 \equiv_P x_2 \Rightarrow x_1 \equiv_Q x_2$ for $x_1,x_2 \in X$.
We will use the following definitions:
\begin{enumerate}
    \item A \emph{reward object} is a function $f : \mathcal{R} \rightarrow X$, where $X$ is any set.
    \item The \emph{ambiguity} $\mathrm{Am}(f)$ of $f$ is the partition of $\mathcal{R}$ given by $R_1 \equiv_f R_2 \iff f(R_1) = f(R_2)$. 
    \item Given a partition $P$ of $\mathcal{R}$, we say that $f$ is \emph{$P$-admissible} if $\mathrm{Am}(f) \preceq P$, i.e.\ $f(R_1) = f(R_2) \Rightarrow R_1 \equiv_P R_2$.
    \item Given a partition $P$ of $\mathcal{R}$, we say that $f$ is \emph{$P$-robust to misspecification} with $g$ if $f$ is $P$-admissible, $f \neq g$, $\mathrm{Im}(g) \subseteq \mathrm{Im}(f)$, and $f(R_1) = g(R_2) \implies R_1 \equiv_P R_2$.
    \item A \emph{reward transformation} is a function $t : \mathcal{R} \rightarrow \mathcal{R}$.
    \item If $F$ and $G$ are sets of reward transformations, then $F \circ G$ is the set of all transformations that can be obtained by composing transformations in $F$ and $G$ arbitrarily, in any order. Note that $F \circ G = G \circ F$.
\end{enumerate}
We will now explain and justify each of these definitions. First of all, anything that can be computed from a reward function can be seen as a reward object. For example, we could consider a  function $b$ that, given a reward $R$, returns the Boltzmann-rational policy with temperature $\beta$ in the MDP $\MDP$, or a  function $r$ that, from $R$, gives the return function $G$ in the MDP $\MDP$. This makes reward objects a versatile abstract building block for more complex constructions. We will mainly, but not exclusively, consider reward objects with the type $\mathcal{R} \rightarrow \Pi$, i.e.\ functions that compute policies from rewards. 

We can use reward objects to create an abstract model of a reward learning algorithm $\mathcal{L}$ as follows; first, we assume, as reasonable, that there is a true underlying reward function $R^\star$, and that the observed training data is generated by a reward object $g$, so that $\mathcal{L}$ observes $g(R^\star)$. Here $g(R^\star)$ could be a \emph{distribution}, which models the case where $\mathcal{L}$ observes a sequence of random samples from some source,
but it could also be a single, finite object. Next, we suppose that $\mathcal{L}$ has a model $f$ of how the observed data relates to $R^\star$, where $f$ is also a reward object, and that $\mathcal{L}$ learns (or converges to) a reward function $R_H$ such that $f(R_H) = g(R^\star)$. If $f \neq g$ then $f$ is \emph{misspecified}, otherwise $f$ is correctly specified. Note that this primarily is a model of the \emph{asymptotic} behaviour of learning algorithms, in the limit of \emph{infinite data}.

There are two ways to interpret $\mathrm{Am}(f)$. First, we can see it as a bound on the amount of information we can get about $R^\star$ by observing (samples from) $f(R^\star)$. For example, multiple reward functions might result in the same Boltzmann-rational policy. Thus, observing trajectories from that policy could never let us distinguish between them: this ambiguity is described by $\mathrm{Am}(b)$. We can also see $\mathrm{Am}(f)$ as the amount of information we need to have about $R^\star$ to construct $f(R^\star)$. 
Next, if $\mathrm{Am}(f) \preceq \mathrm{Am}(g)$ and $f \neq g$, this means that we get less information about $R^\star$ by observing $g(R^\star)$ than $f(R^\star)$, and that we would need more information to construct $f(R^\star)$ than $g(R^\star)$. For an extensive discussion about these notions, see \citet{skalse2022}.

Intuitively, we want to say that a behavioural model is robust to some type of misspecification if an algorithm based on that model will learn a reward function that is \enquote{close enough} to the true reward function when subject to that misspecification. To formalise this intuitive statement, we first need a definition of what it should mean for two reward functions to be \enquote{close enough}. In this work, we have chosen to define this in terms of \emph{equivalence classes}. Specifically, we assume that we have a partition $P$ of $\mathcal{R}$ (which, of course, corresponds to an equivalence relation), and that the learnt reward function $R_H$ is \enquote{close enough} to the true reward $R^\star$ if they are in the same class, $R_H \equiv_P R^\star$. We will for now leave open the question of which partition $P$ of $\mathcal{R}$ to pick, and later revisit this question in Section~\ref{section:R_equivalence_classes}.

Given this, we can now see that our definition of $P$-admissibility is equivalent to stating that a learning algorithm $\mathcal{L}$ based on $f$ is guaranteed to learn a reward function that is $P$-equivalent to the true reward function when there is no misspecification. 
Furthermore, our definition of $P$-robustness says that $f$ is $P$-robust to misspecification with $g$ if any learning algorithm $\mathcal{L}$ based on $f$ is guaranteed to learn a reward function that is $P$-equivalent to the true reward function when trained on data generated from $g$.
The requirement that $\mathrm{Im}(g) \subseteq \mathrm{Im}(f)$ ensures that the learning algorithm $\mathcal{L}$ is never given data that is impossible according to its model. 
Depending on how $\mathcal{L}$ reacts to such data, it may be possible to drop this requirement.
We include it, since we want our analysis to apply to all algorithms. 
The requirement that $f$ is $P$-admissible is included to rule out some uninteresting edge cases.


Reward transformations can be used to characterise the ambiguity of reward objects, or define other partitions of $\mathcal{R}$. 
Specifically, we say that a partition $P$ corresponds to a set of reward transformations $T_P$ if $T_P$ contains all reward transformations $t$ that satisfy $t(R) \equiv_P R$.
If $P$ is the ambiguity of $f$ then $T_P$ would be the set of all reward transformations that satisfy $f(R) = f(t(R))$.

\subsection{Fundamental Lemmas}\label{section:lemmas}

We here give two fundamental lemmas that we will later use to prove our core results. These lemmas can also be used to easily derive the misspecification robustness of new models, beyond those considered in this work. All of our proofs are provided in the supplementary material, which also contains several additional results about our framework.





\begin{lemma}\label{lemma:non-robustness_inherited}
If $f$ is not $P$-robust to misspecification with $g$, and $\mathrm{Im}(g) \subseteq \mathrm{Im}(f)$, then for any $h$, $h \circ f$ is not $P$-robust to misspecification with $h \circ g$.
\end{lemma}

This lemma states that if we have an object $h \circ f$ that can be computed from some intermediary object $f$, and $f$ is not $P$-robust to some form of misspecification, then $h \circ f$ is likewise not robust to the corresponding misspecification. In other words, any misspecification that $f$ is sensitive to, is \enquote{inherited} by all objects that can be computed from $f$.

\begin{lemma}\label{lemma:how_to_calculate_robustness_sets}
If $f$ is $P$-admissible, 
and $T$ is the set of all reward transformations that preserve $P$, then $f$ is $P$-robust to misspecification with $g$ if and only if $g = f \circ t$ for some $t \in T$ where $f \circ t \neq f$.
\end{lemma}

This lemma gives us a very powerful tool for characterising the misspecification robustness of reward objects. Specifically, we can derive the set of objects to which $f$ is $P$-robust by first deriving the set $T$ of all transformations that preserve $P$, and then composing $f$ with each $t \in T$.






\subsection{Reward Transformations}


We here introduce several classes of reward transformations, that we will later use to express our results.  
First recall \emph{potential shaping}
\cite{ng1999}:
%
%
%
\begin{definition}[Potential Shaping]
A \emph{potential function} is a function $\Phi : \States \to \mathbb{R}$.
Given a discount $\discount$,
we say that
  $R_2 \in \mathcal{R}$ is produced by \emph{potential shaping} of $R_1 \in \mathcal{R}$
if for some potential $\Phi$,
  $$
  R_2(s,a,s') = R_1(s,a,s') + \discount\cdot\Phi(s') - \Phi(s).
  $$
\end{definition}
Potential shaping is widely used for reward shaping.
We next define two classes of transformations that were used by \citet{skalse2022}, starting with \emph{$S'$-redistribution}.
\begin{definition}[$S'$-Redistribution]
\label{def:sprime-redistribution}
Given a transition function $\TransitionDistribution$,
we say that
    $R_2 \in \mathcal{R}$ is produced by \emph{$S'$-redistribution} of $R_1 \in \mathcal{R}$
if
    $$
    \Expect{S' \sim \TransitionDistribution(s,a)}{\reward_1(s,a,S')}
    = \Expect{S' \sim \TransitionDistribution(s,a)}{\reward_2(s,a,S')}.
    $$ 
\end{definition}


If $s_1$, $s_2 \in \mathrm{Supp}(\TransitionDistribution(s,a))$ then $S'$-redistribution can increase $\reward(s,a,s_1)$ if it decreases $\reward(s,a,s_2)$ proportionally. $S'$-redistribution can also change $\reward$ arbitrarily for transitions that occur with probability 0. We next consider \emph{optimality-preserving transformations}:
%

\begin{definition}\label{def:optimality_preserving_transformation}
Given a transition function $\TransitionDistribution$ and a discount $\discount$, we say that $R_2 \in \mathcal{R}$ is produced by an \emph{optimality-preserving transformation} of $R_1 \in \mathcal{R}$ if there exists a function $\psi : \States \rightarrow \mathbb{R}$ such that
$$
\mathbb{E}_{S' \sim \TransitionDistribution(s,a)}[R_2(s,a,S') + \discount\cdot\psi(S')] \leq \psi(s), 
$$
with equality if and only if $a \in \mathrm{argmax}_{a \in \Actions} A^\star_1(s,a)$. 
\end{definition}

An optimality preserving transformation of $R_1$ lets us pick an arbitrary new value function $\psi$, and then adjust $R_2$ in any way that respects the new value function and the argmax of $A^\star_1$ --- the latter condition ensures that the same actions (and hence the same policies) stay optimal. 


Based on these definitions, we can now specify several \emph{sets} of reward transformations:
\begin{enumerate}
    \item Let $\PS$ be the set of all  reward transformations $t$ such that $t(R)$ is given by potential shaping of $R$ relative to the discount $\discount$.
    \item Let $\SR$ be the set of all reward transformations $t$ such that $t(R)$ is given by $S'$-redistribution of $R$ relative to the transition function $\TransitionDistribution$.
    \item Let $\LS$ be the set of all reward transformations $t$ that scale each reward function by some positive constant, i.e.\ for each $R$ there is a $c \in \mathbb{R}^+$ such that $t(R)(s,a,s') = c \cdot R(s,a,s')$.
    \item Let $\CS$ be the set of all reward transformations $t$ that shift each reward function by some constant, i.e.\ for each $R$ there is a $c \in \mathbb{R}$ such that $t(R)(s,a,s') = R(s,a,s') + c$.
    \item Let $\OPT$ be the set of all reward transformations $t$ such that $t(R)$ is given by an optimality-preserving transformation of $R$ relative to $\TransitionDistribution$ and $\discount$.
\end{enumerate}
Note that these sets are defined in a way that allows their transformations to be \enquote{sensitive} to the reward function it takes as input. For example, a transformation $t \in \mathrm{PS}_\gamma$ might apply one potential function $\Phi_1$ to $R_1$, and a different potential function $\Phi_2$ to $R_2$. Similarly, a transformation $t \in \mathrm{LS}$ might scale $R_1$ by a positive constant $c_1$, and $R_2$ by a different constant $c_2$, etc.
Note also that $\CS \subseteq \PS$ (for all $\gamma$), and that all sets are subsets of $\OPT$ \cite[see][]{skalse2022}.

\subsection{Two Equivalence Classes for Reward Functions}\label{section:R_equivalence_classes}

Our definition of misspecification robustness is given relative to an equivalence relation on $\mathcal{R}$.
In this section, we define two important equivalence classes, and characterise the transformations that preserve them. Our later results will be given relative to these two equivalence classes.

Given an environment $\M = \langle \States,\Actions,\TransitionDistribution,\InitStateDistribution,\_,\discount\rangle$
and two reward functions $R_1$, $R_2$, we say that $R_1 \equiv_\mathrm{OPT^\M} R_2$ if 
$\langle \States,\Actions,\TransitionDistribution,\InitStateDistribution,R_1,\discount\rangle$ and $\langle \States,\Actions,\TransitionDistribution,\InitStateDistribution,R_2,\discount\rangle$
have the same \emph{optimal} policies, and that $R_1 \equiv_\mathrm{ORD^\M} R_2$ if they have the same \emph{ordering} of policies.
\footnote{By this, we mean that $\Evaluation_1(\pi) > \Evaluation_1(\pi')$ if and only if $\Evaluation_2(\pi) > \Evaluation_2(\pi')$, for all pairs of policies $\pi, \pi'$.} 
Note that if $R_1 \equiv_\mathrm{ORD^\M} R_2$ then $R_1 \equiv_\mathrm{OPT^\M} R_2$.
\citet{skalse2022} showed that $R \equiv_\mathrm{OPT^\M} t(R)$ for all $R$ if and only if $t \in \OPT$ (their Theorem 3.16). We characterise the transformations that preserve $\mathrm{ORD^\M}$, which is a novel contribution.

\begin{theorem}\label{thm:order_of_policies}
$R_1 \equiv_\mathrm{ORD^\M} R_2$ if and only if $R_2 = t(R_1)$ for some $t \in \SR \circ \PS \circ \LS$.
\end{theorem}



Stated differently, Theorem~\ref{thm:order_of_policies} is saying that the MDPs $(\States, \Actions, \TransitionDistribution, \InitStateDistribution, 
R_1, \discount)$ and $(\States, \Actions, \TransitionDistribution, \InitStateDistribution, 
R_2, \discount)$ have the same ordering of policies if and only if $R_1$ and $R_2$ differ by potential shaping (with $\discount$), positive linear scaling, and $S'$-redistribution (with $\TransitionDistribution$), applied in any order.


$\mathrm{OPT^\M}$ and $\mathrm{ORD^\M}$ are two equivalence relations that should be relevant and informative in almost any context, which is why we have chosen to carry out our analysis in terms of these two relations. However, other partitions could be selected instead. For example, if we know that the learnt reward $R_H$ will be used to compute a reward object $f$, then $\mathrm{Am}(f)$ would be a natural choice.

We now have results for reasoning about misspecification robustness in IRL. In particular, Lemma~\ref{lemma:how_to_calculate_robustness_sets} tells us that if we want to find the functions that $f$ is $P$-robust to misspecification with, then all we need to do is find the reward transformations that preserve $P$, and then compose them with $f$. $\mathrm{OPT^\M}$ and $\mathrm{ORD^\M}$ are reasonable choices of $P$, and the transformations that preserve them were just provided.

\section{Misspecification Robustness of IRL Models}\label{section:misspecification_robustness}

We here give our main results on the misspecification robustness of IRL, looking both at misspecification of the behavioural model, as well as of the MDP.

\subsection{Misspecified Behavioural Models}\label{section:misspecified_policies}

Let $\Pi^+$ be the set of all policies such that $\pi(a \mid s) > 0$ for all $s,a$, let $\M = \langle \States,\Actions,\TransitionDistribution,\InitStateDistribution,\_,\discount\rangle$, and let $F^\M$ be the set of all functions $f^\M : \mathcal{R} \rightarrow \Pi^+$ 
that, given $R$, returns a policy $\pi$ which satisfies
\begin{equation*}
    \mathrm{argmax}_{a \in \Actions} \pi(a \mid s) = \mathrm{argmax}_{a \in \Actions} \QStar(s,a), 
\end{equation*}
where $\QStar$ is the optimal $Q$-function in $\MDP$.
In other words, $F^\M$ is the set of functions that generate policies which take each action with positive probability, and that take the optimal actions with the highest probability. 
This class is quite large, and includes e.g.\ Boltzmann-rational policies (for any $\beta$), but it does not include optimal policies (since they do not take all actions with positive probability) or causal entropy maximising policies (since they may take suboptimal actions with high probability).



\begin{theorem}\label{thm:boltzmann_weak_robustness}
Let $f^\M \in F^\M$ be surjective onto $\Pi^+$. Then $f^M$ is $\mathrm{OPT^\M}$-robust to misspecification with $g$ if and only if $g \in F^\M$ and $g \neq f^\M$.
\end{theorem}

Boltzmann-rational policies are surjective onto $\Pi^+$,\footnote{If a policy $\pi$ takes each action with positive probability, then its action probabilities are always the softmax of some $Q$-function, and any $Q$-function corresponds to some reward function.} so Theorem~\ref{thm:boltzmann_weak_robustness} exactly characterises the misspecification to which the Boltzmann-rational model is $\mathrm{OPT^\M}$-robust. 

Let us briefly comment on the requirement that $\pi(a \mid s) > 0$, which corresponds to the condition that $\mathrm{Im}(g) \subseteq \mathrm{Im}(f)$ in our definition of misspecification robustness. If a learning algorithm $\mathcal{L}$ is based on a model $f : \mathcal{R} \rightarrow \Pi^+$ then it assumes that the observed policy takes each action with positive probability in every state. What happens if such an algorithm $\mathcal{L}$ is given data from a policy that takes some action with probability $0$? This depends on $\mathcal{L}$, but for most sensible algorithms the result should simply be that $\mathcal{L}$ assumes that those actions are taken with a positive but low probability. This means that it should be possible to drop the requirement that $\pi(a \mid s) > 0$ for most reasonable algorithms. 

We next turn our attention to the misspecification to which the Boltzmann-rational model is $\mathrm{ORD^\M}$-robust. Let $\psi : \mathcal{R} \rightarrow \mathbb{R}^+$ be any function from reward functions to positive real numbers, and let $b_\psi^\M : \mathcal{R} \rightarrow \Pi^+$ be the function that, given $R$, returns the Boltzmann-rational policy with temperature $\psi(R)$ in $\MDP$. Moreover, let $B^\M = \{b_\psi^\M : \psi \in \mathcal{R} \rightarrow \mathbb{R}^+\}$ be the set of all such functions $b_\psi^\M$. 
This set includes Boltzmann-rational policies; just let $\psi$ return a constant $\beta$ for all $R$.

\begin{theorem}
If $b_\psi^\M \in B^\M$ then $b_\psi^\M$ is $\mathrm{ORD^\M}$-robust to misspecification with $g$ if and only if $g \in B^\M$ and $g \neq b_\psi^\M$.
\end{theorem}

This means that the Boltzmann-rational model is $\mathrm{ORD^\M}$-robust to misspecification of the temperature parameter $\beta$, but not to any other form of misspecification.


We next turn our attention to optimal policies. First of all, a policy is optimal if and only if it only gives support to optimal actions,
and if an optimal policy gives support to multiple actions in some state, then we would normally not expect the exact probability it assigns to each action to convey any information about the reward function. We will therefore only look at the actions that the optimal policy takes, and ignore the relative probability it assigns to those actions. 
Formally, we will treat optimal policies as functions $\OptimalPolicy : \States \rightarrow \mathcal{P}(\mathrm{argmax}_{a \in \Actions} \AStar)-\{\varnothing\}$; i.e.\ as functions that for each state return a non-empty subset of the set of all actions that are optimal in that state.
Let $\mathcal{O}^\M$ be the set of all functions that return such policies,
and let $o_m^\M \in \mathcal{O}^\M$ be the function that, given $R$, returns the function that maps each state to the set of \emph{all} actions which are optimal in that state. Intuitively, $o_m^\M$ corresponds to optimal policies that take all optimal actions with positive probability.

\begin{theorem}\label{thm:optimal_policy_robustness}
No function in $\mathcal{O}^\M$ is $\mathrm{ORD^\M}$-admissible. 
The only function in $\mathcal{O}^\M$ that is $\mathrm{OPT^\M}$-admissible is $o_m^\M$, but $o_m^\M$ is not $\mathrm{OPT^\M}$-robust to any misspecification.
\end{theorem}






This essentially means that the optimality model is not robust to any form of misspecification.
We finally turn our attention to causal entropy maximising policies. 
As before, let $\psi : \mathcal{R} \rightarrow \mathbb{R}^+$ be any function from reward functions to positive real numbers, and let $c_\psi^\M : \mathcal{R} \rightarrow \Pi^+$ be the function that, given $R$, returns the causal entropy maximising policy with weight $\psi(R)$ in $\MDP$. Furthermore, let $C^\M = \{c_\psi^M : \psi \in \mathcal{R} \rightarrow \mathbb{R}^+\}$ be the set of all such functions $c_\psi^\M$. 
This set includes causal entropy maximising policies; just let $\psi$ return a constant $\alpha$ for all $R$.

\begin{theorem}\label{thm:mce_strong_robustness}
If $c_\psi^\M \in C^\M$ then $c_\psi^\M$ is $\mathrm{ORD^\M}$-robust to misspecification with $g$ if and only if $ g \in C^\M$ and $g \neq c_\psi^\M$.
\end{theorem}

In other words, the maximal causal entropy model is $\mathrm{ORD^\M}$-robust to misspecification of the weight $\alpha$, but not to any other kind of misspecification.

Finally, let us briefly discuss the misspecification to which the maximal causal entropy model is $\mathrm{OPT^\M}$-robust. Lemma~\ref{lemma:how_to_calculate_robustness_sets} tells us that $c_\psi^\M \in C^\M$ is $\mathrm{OPT^\M}$-robust to misspecification with $g$ if $g = c_\psi^\M \circ t$ for some $t \in \OPT$.
In other words, if $g(R_1) = \pi$ then there must exist an $R_2$ such that $\pi$ maximises causal entropy with respect to $R_2$, and such that $R_1$ and $R_2$ have the same optimal policies. It seems hard to express this as an intuitive property of $g$, so we have refrained from stating this result as a theorem.


\subsection{Misspecified MDPs}\label{section:misspecified_mdps}

A reward object 
can be parameterised by a $\gamma$ or $\TransitionDistribution$, implicitly or explicitly. 
For example, the reward objects in Section~\ref{section:misspecified_policies} are parameterised by $\M = \langle\States,\Actions,\TransitionDistribution,\InitStateDistribution,\_,\gamma\rangle$.
In this section, we explore what happens if these parameters are misspecified.
We show that nearly all behavioural models are sensitive to this type of misspecification.

Theorems~\ref{thm:boltzmann_weak_robustness}-\ref{thm:mce_strong_robustness} already tell us that the standard behavioural models are not ($\mathrm{ORD^\M}$ or $\mathrm{OPT^\M}$) robust to misspecified $\discount$ or $\TransitionDistribution$, since the sets $F^\M$, $B^\M$, and $C^\M$, all are parameterised by $\discount$ and $\TransitionDistribution$.
We will generalise this further.
To do this, we first derive two lemmas.
We say that $\TransitionDistribution$ is \emph{trivial} if for each $s \in \States$, $\TransitionDistribution(s,a) = \TransitionDistribution(s,a')$ for all $a,a' \in \Actions$.

\begin{lemma}\label{lemma:no_t_robustness}
If $f^{\TransitionDistribution_1} = f^{\TransitionDistribution_1} \circ t$ for all $t \in S'\mathrm{R}_{\TransitionDistribution_1}$ then $f^{\TransitionDistribution_1}$ is not $\mathrm{OPT}^\M$-admissible for $\M = \langle \States,\Actions,\TransitionDistribution_2, \InitStateDistribution, \_, \gamma \rangle$ unless $\TransitionDistribution_1 = \TransitionDistribution_2$.
\end{lemma}

\begin{lemma}\label{lemma:no_g_robustness}
If $f^{\gamma_1} = f^{\gamma_1} \circ t$ for all $t \in \mathrm{PS}_{\gamma_1}$ then $f^{\gamma_1}$ is not $\mathrm{OPT}^\M$-admissible for $\M = \langle \States,\Actions,\TransitionDistribution, \InitStateDistribution, \_, \gamma_2 \rangle$ unless $\gamma_1 = \gamma_2$ or $\TransitionDistribution$ is trivial.
\end{lemma}

Note that if $f$ is not $\mathrm{OPT^\M}$-admissible then $f$ is also not $\mathrm{ORD^\M}$-admissible.
From these lemmas, together with Lemma~\ref{lemma:non-robustness_inherited},
we get the following result:

\begin{theorem}\label{thm:IRL_lacks_g_t_robustness}
If $f^{\TransitionDistribution_1} = f^{\TransitionDistribution_1} \circ t$ for all $t \in S'\mathrm{R}_{\TransitionDistribution_1}$ and $f^{\TransitionDistribution_2} = f^{\TransitionDistribution_2} \circ t$ for all $t \in S'\mathrm{R}_{\TransitionDistribution_2}$, then $f^{\TransitionDistribution_1}$ is not $\mathrm{OPT}^\M$-robust to misspecification with $f^{\TransitionDistribution_2}$ for any $\M$.
Moreover, if $f^{\gamma_1} = f^{\gamma_1} \circ t$ for all $t \in \mathrm{PS}_{\gamma_1}$ and $f^{\gamma_2} = f^{\gamma_2} \circ t$ for all $t \in \mathrm{PS}_{\gamma_2}$, then $f^{\gamma_1}$ is not $\mathrm{OPT}^\M$-robust to misspecification with $f^{\gamma_2}$ for any $\M$ whose transition function $\TransitionDistribution$ is non-trivial.
\end{theorem}

In other words, if a behavioural model is insensitive to $S'$-redistribution, then that model is not $\mathrm{OPT}^\M$-robust (and therefore also not $\mathrm{ORD}^\M$-robust) to misspecification of the transition function $\tau$. 
Similarly, if the behavioural model is insensitive to potential shaping, then that model is not $\mathrm{OPT}^\M$-robust (and therefore also not $\mathrm{ORD}^\M$-robust) to misspecification of the discount parameter $\gamma$. 
Note that all transformations in $\SR$ and $\PS$ preserve the ordering of policies.
This means that an IRL algorithm must specify $\tau$ and $\gamma$ correctly in order to guarantee that the learnt reward $R_H$ has the same optimal policies as the true underlying reward $R^*$,
unless the algorithm is based on a behavioural model which says that the observed policy depends on features of $R$ which do not affect its policy ordering. 
This should encompass most natural behavioural models.

That being said, we note that this result relies on the requirement that the learnt reward function should have \emph{exactly} the same optimal policies, or ordering of policies, as the true reward function. If $\gamma_1 \approx \gamma_2$ and $\TransitionDistribution_1 \approx \TransitionDistribution_2$, then the learnt reward function's optimal policies and policy ordering will presumably be \emph{similar} to that of the true reward function. 
Analysing this case is beyond the scope of this paper, but we consider it to be an important topic for further work.

\section{Generalising the Analysis}

In this section, we discuss different ways to generalise our results. We consider what happens if $R$ is restricted to a subset of $\mathcal{R}$, what might happen if $R$ is drawn from a known prior distribution and the learning algorithm has a known inductive bias, and whether we can use stronger equivalence classes to guarantee various forms of transfer learning. 

\subsection{Restricted Reward Functions}\label{section:subspaces}

Here, we discuss what happens if the reward function is restricted to belong to some subset of $\mathcal{R}$, i.e.\ if we know that $R \in \Rspace$ for some $\Rspace \subseteq \mathcal{R}$. For example, it is common to consider reward functions that are linear in some state features.
It is also common to define the reward function over a restricted domain, such as $\SxA$; this would correspond to restricting $\mathcal{R}$ to the set of reward functions such that $R(s,a,s') = R(s,a,s'')$ for all $s,a,s',s''$. 
As we will see, our results are largely unaffected by such restrictions.


We first need to generalise the framework, which is straightforward. 
Given partitions $P$, $Q$ of $\mathcal{R}$, reward objects $f$, $g$, and set $\Rspace \subseteq \mathcal{R}$, we say that $P \preceq Q$ on $\Rspace$ if $R_1 \equiv_P R_2$ implies $R_1 \equiv_Q R_2$ for all $R_1,R_2 \in \Rspace$, that $f$ is $P$-admissible on $\Rspace$ if $\mathrm{Am}(f) \preceq P$ on $\Rspace$, and that $f$ is $P$-robust to misspecification with $g$ on $\Rspace$ if $f$ is $P$-admissible on $\Rspace$, $f|_\Rspace \neq g|_\Rspace$, $\mathrm{Im}(g|_\Rspace) \subseteq \mathrm{Im}(f|_\Rspace)$, and $f(R_1) = g(R_2) \implies R_1 \equiv_P R_2$ for all $R_1,R_2 \in \Rspace$. 

All lemmas in Section~\ref{section:lemmas} apply with these more general definitions for any arbitrary subset $\Rspace \subseteq \mathcal{R}$. 
Moreover, the theorems in Section~\ref{section:misspecification_robustness} also carry over very directly:


\begin{theorem}\label{thm:subspace}
If $f$ is $P$-robust to misspecification with $g$ on $\Rspace$ then $f$ is $P$-robust to misspecification with $g'$ on $\mathcal{R}$ for some $g'$ where $g'|_\Rspace = g|_\Rspace$, unless $f$ is not $P$-admissible on $\mathcal{R}$. If $f$ is $P$-robust to misspecification with $g$ on $\mathcal{R}$ then $f$ is $P$-robust to misspecification with $g$ on $\Rspace$, unless $f|_\Rspace = g|_\Rspace$.
\end{theorem}

The intuition for this theorem is that if $f$ is $P$-robust to misspecification with $g$ if and only if $g \in G$, then $f$ is $P$-robust to misspecification with $g'$ on $\Rspace$ if and only if $g'$ behaves like some $g \in G$ for all $R \in \Rspace$. Restricting $\mathcal{R}$ does therefore not change the problem in any significant way.










If an equivalence relation $P$ of $\mathcal{R}$ is characterised by a set of reward transformations $T$, then the corresponding equivalence relation on $\Rspace$ is characterised by the set of reward transformations $\{t \in T : \mathrm{Im}(t|_\Rspace) \subseteq \Rspace\}$; this can be used to generalise Theorem~\ref{thm:order_of_policies}.
However, here there is a minor subtlety to be mindful of: $(A \circ B) - C$ is not necessarily equal to $(A - C) \circ (B - C)$. 
This means that if we wish to specify $\{t \in A \circ B : \mathrm{Im}(t|_\Rspace) \subseteq \Rspace\}$, then we cannot do this by simply removing the transformations where $\mathrm{Im}(t|_\Rspace) \not\subseteq \Rspace$ from each of $A$ and $B$. 
For example, consider the transformations $\SR \circ \PS$ restricted to the space $\Rspace$ of reward functions where $R(s,a,s') = R(s,a,s'')$, i.e.\ to reward functions over the domain $\SxA$. The only transformation in $\SR$ on $\Rspace$ is the identity mapping, and the only transformations in $\PS$ on $\Rspace$ are those where $\Phi$ is constant over all states.
However, $\SR \circ \PS$ on $\Rspace$ contains all transformations where $\Phi$ is selected arbitrarily, and $t(R)(s,a,s')$ is set to $R(s,a,s') + \gamma\Expect{}{\Phi(S')} - \Phi(s)$.
This means that there probably are no general shortcuts for deriving $\{t \in T : \mathrm{Im}(t|_\Rspace) \subseteq \Rspace\}$ for arbitrary $\Rspace$.

It should be noted that our \emph{negative} results (i.e., those in Section~\ref{section:misspecified_mdps}) might \emph{not} hold if $\mathcal{R}$ is restricted. 
Recall that $f$ is not $P$-robust to misspecification with $g$ if there exist $R_1, R_2$ such that $g(R_1) = f(R_2)$, but $R_1 \not\equiv_P R_2$. If $\mathcal{R}$ is restricted, it could be the case that all such counterexamples are removed.
For example, if we restrict $\mathcal{R}$ to e.g.\ the set $\Rspace$ of reward functions that only reward a single transition, then Lemma~\ref{lemma:no_g_robustness}, and the corresponding part of Theorem~\ref{thm:IRL_lacks_g_t_robustness}, no longer apply.\footnote{The reason for this is that there are no $R_1, R_2 \in \Rspace$ where $R_1 = t(R_2)$ for some $t \in \PS$.}
This means that, if the reward function is guaranteed to lie in this set $\Rspace$, then a behavioural model may still be $\mathrm{OPT}^\M$-robust to a misspecified discount parameter. However, the reason for this is simply that the discount parameter no longer affects which policies are optimal if there is only a single transition that has non-zero reward.

\subsection{Known Prior and Inductive Bias}\label{section:known_probability_distributions}

So far, we have assumed that we do not know which distribution $R$ is sampled from, or which inductive bias the learning algorithm $\mathcal{L}$ has. In this section, we discuss what might happen if we lift these assumptions.

To some extent, our results in Section~\ref{section:subspaces} can be used to understand this setting as well. 
Suppose we have a set $\Rspace \subseteq \mathcal{R}$ of \enquote{likely} reward functions, such that $\mathbb{P}(R^\star \in \Rspace) = 1-\delta$, and such that the learning algorithm $\mathcal{L}$ returns a reward function $R_H$ in $\Rspace$ if there exists an $R_H \in \Rspace$ such that $f(R_H) = g(R^\star)$. 
Then if $f$ is $P$-robust to misspecification with $g$ on $\Rspace$, it follows that $\mathcal{L}$ returns an $R_H$ such that $R_H \equiv_P R^\star$ with probability at least $1-\delta$. 

So, for example, suppose $\Rspace$ is the set of all reward functions that are \enquote{sparse}, for some way of formalising that property. Then this tells us, informally, that if the underlying reward function is likely to be sparse, and if $\mathcal{L}$ will attempt to fit a sparse reward function to its training data, then it is sufficient that $f$ is $P$-robust to misspecification with $g$ on the set of all sparse reward functions, to ensure that the learnt reward function $R_H$ is $P$-equivalent to the true reward function with high probability. 
It seems likely that more specific claims could be made about this setting, but we leave such analysis as a topic for future work.

\subsection{Transfer to New Environments}\label{section:transfer_learning}

The equivalence relations we have worked with ($\mathrm{OPT}^\M$ and $\mathrm{ORD}^\M$) only guarantee that the learnt reward function $R_H$ has the same optimal policies, or ordering of policies, as the true reward $R^\star$ in a given environment $\M = \langle \States,\Actions,\TransitionDistribution,\InitStateDistribution,\_,\discount\rangle$. A natural question is what happens if we strengthen this requirement, and demand that $R_H$ has the same optimal policies, or ordering of policies, as $R^\star$, for any choice of $\TransitionDistribution$, $\InitStateDistribution$, or $\discount$. We discuss this setting here.

In short, it is impossible to guarantee transfer to any $\TransitionDistribution$ or $\gamma$ within our framework, and trivial to guarantee transfer to any $\InitStateDistribution$. First, the lemmas provided in Section~\ref{section:misspecified_mdps} tell us that none of the standard behavioural models are $\mathrm{OPT}^\M$-admissible when $\TransitionDistribution$ or $\gamma$ is different from that of the training environment. This means that none of them can guarantee that $R_H$ has the same optimal policies (or ordering of policies) as $R^\star$ if $\TransitionDistribution$ or $\gamma$ is changed, with or without misspecification. Second, if $R_1 \equiv_{\mathrm{ORD}^\M} R_2$ or $R_1 \equiv_{\mathrm{OPT}^\M} R_2$, then this remains the case if $\InitStateDistribution$ is changed. We can thus trivially guarantee transfer to arbitrary $\InitStateDistribution$.

We would also like to remark on a subtlety regarding Theorem~\ref{thm:order_of_policies}. One might expect that two reward functions $R_1$ and $R_2$ must have the same policy ordering for all $\tau$ if and only if they differ by potential shaping and linear scaling. However, this is not the case. To see this, consider the rewards $R_1$, $R_2$ where $R_1(s_1,a_1,s_1) = 1$, $R_1(s_1,a_1,s_2) = 0.5$, $R_2(s_1,a_1,s_1) = 0.5$, and $R_2(s_1,a_1,s_2) = 1$, and where $R_1$ and $R_2$ are $0$ for all other transitions. Now $R_1$ and $R_2$ do not differ by potential shaping and linear scaling, yet they have the same policy order for all $\tau$.

\section{Discussion}

In this section, we discuss the implications of our results, as well as their limitations.

\subsection{Conclusions and Implications}

We have shown that the misspecification robustness of the behavioural models in IRL can be quantified and understood.
Our results show that the Boltzmann-rational model is substantially more robust to misspecification than the optimality model; the optimality model is not robust to any misspecification, whereas the Boltzmann-rationality model is at least $\mathrm{OPT}^\M$-robust to many kinds of misspecification. 
This is not necessarily unexpected,
but we now have formal guarantees to back this intuition. 
We have also quantified the misspecification robustness of the maximal causal entropy model, and found that it lies somewhere between that of the Boltzmann-rational model and the optimality model.

We have shown that none of the standard models are robust to a misspecified $\TransitionDistribution$ or $\gamma$. Moreover, we need to make very minimal assumptions about how the demonstrator policy is computed to obtain this result, which means that it is likely to generalise to new behavioural models as well. 
We find this quite surprising; the discount $\gamma$ is typically selected in a somewhat arbitrary way, and it can often be difficult to establish post-facto which $\gamma$ was used to compute a given policy. 
The fact that $\TransitionDistribution$ must be specified correctly is somewhat less surprising, yet important to have established. 


In addition to these contributions, we have also provided several formal tools for deriving the misspecification robustness of new behavioural models, in the form of the lemmas in Section~\ref{section:lemmas}.
In particular, if we have a model $f$, and we wish to use the learnt reward to compute an object $g$, then we can obtain an expression of the set of all functions to which $f$ is robust in the following way; first, derive $\mathrm{Am}(g)$, and then characterise this partition of $\mathcal{R}$ using a set of reward transformations $T$. Then, as per Lemma~\ref{lemma:how_to_calculate_robustness_sets}, we can obtain the functions that $f$ is robust to misspecification with by simply composing $f$ with each $t \in T$.
If we want to know which functions $f$ is robust to misspecification with in a strong sense, then we can obtain an informative answer to this question by composing $f$ with the transformations that preserve the ordering of all policies, which in turn is provided by Theorem~\ref{thm:order_of_policies}.
Lemma~\ref{lemma:non-robustness_inherited} also makes it easier to intuitively reason about the robustness properties of various kinds of behavioural models.


\subsection{Limitations and Further Work}

Our analysis makes a few simplifying assumptions, that could be ideally lifted in future work. First of all, we have been working with \emph{equivalence relations} on $\mathcal{R}$, where two reward functions are either equivalent or not. It might be fruitful to instead consider \emph{distance metrics} on $\mathcal{R}$: this could make it possible to obtain results such as e.g.\ bounds on the distance between the true reward function and the learnt reward function, given various forms of misspecification. 
We believe it would be especially interesting to re-examine Theorem~\ref{thm:IRL_lacks_g_t_robustness} through this lens.

Another notable direction for extensions could be to further develop the analysis in Section~\ref{section:known_probability_distributions}, and study the misspecification robustness of different behavioural models in the context where we have particular, known priors concerning $R$.
Our comments on this setting are fairly preliminary, and it might be possible to draw additional, interesting conclusions if this setting is explored more extensively.

Moreover, we have studied the behaviour of algorithms in the limit of infinite data, under the assumption that this is similar to their behaviour in the case of finite but sufficiently large amounts of data. Therefore, another possible extension could be to more rigorously examine the properties of these models in the case of finite data. 

Finally, our analysis has of course been limited to the behavioural models that are currently most popular in IRL (optimality, Boltzmann rationality, and causal entropy maximisation) and two particular equivalence relations ($\mathrm{OPT^\M}$ and $\mathrm{ORD^\M}$). Another direction for extensions would be to broaden our analysis to larger classes of models, and perhaps also to more equivalence relations.
In particular, it would be interesting to analyse more realistic behavioural models, which incorporate e.g.\ prospect theory \cite{prospecttheory} or hyperbolic discounting.












\bibliography{aaai23.bib}

\begin{thebibliography}{28}
\providecommand{\natexlab}[1]{#1}

\bibitem[{Abbeel, Coates, and Ng(2010)}]{abbeel2010}
Abbeel, P.; Coates, A.; and Ng, A.~Y. 2010.
\newblock Autonomous Helicopter Aerobatics Through Apprenticeship Learning.
\newblock \emph{The International Journal of Robotics Research}, 29(13):
  1608--1639.

\bibitem[{Armstrong and Mindermann(2019)}]{armstrong2019occams}
Armstrong, S.; and Mindermann, S. 2019.
\newblock Occam's razor is insufficient to infer the preferences of irrational
  agents.
\newblock arXiv:1712.05812.

\bibitem[{Cao, Cohen, and Szpruch(2021)}]{cao2021}
Cao, H.; Cohen, S.~N.; and Szpruch, L. 2021.
\newblock Identifiability in Inverse Reinforcement Learning.
\newblock \emph{arXiv preprint}, arXiv:2106.03498 [cs.LG].

\bibitem[{Chan, Critch, and Dragan(2019)}]{IrrationalityCanHelp}
Chan, L.; Critch, A.; and Dragan, A. 2019.
\newblock Irrationality can help reward inference.

\bibitem[{Dvijotham and Todorov(2010)}]{dvijotham2010}
Dvijotham, K.; and Todorov, E. 2010.
\newblock Inverse Optimal Control with Linearly-Solvable {MDPs}.
\newblock In \emph{Proceedings of the 27th International Conference on Machine
  Learning}, 335--342. Haifa, Israel: Omnipress, Madison, Wisconsin, USA.

\bibitem[{Evans, Stuhlmueller, and Goodman(2015)}]{IgnorantAndInconsistent}
Evans, O.; Stuhlmueller, A.; and Goodman, N.~D. 2015.
\newblock Learning the Preferences of Ignorant, Inconsistent Agents.
\newblock arXiv:1512.05832.

\bibitem[{Freedman, Shah, and Dragan(2021)}]{choicesetmisspecification}
Freedman, R.; Shah, R.; and Dragan, A. 2021.
\newblock Choice Set Misspecification in Reward Inference.

\bibitem[{Gleave et~al.(2020)Gleave, Dennis, Legg, Russell, and Leike}]{epic}
Gleave, A.; Dennis, M.; Legg, S.; Russell, S.; and Leike, J. 2020.
\newblock Quantifying Differences in Reward Functions.

\bibitem[{Hadfield-Menell et~al.(2016)Hadfield-Menell, Russell, Abbeel, and
  Dragan}]{CIRL}
Hadfield-Menell, D.; Russell, S.~J.; Abbeel, P.; and Dragan, A. 2016.
\newblock Cooperative Inverse Reinforcement Learning.
\newblock In Lee, D.; Sugiyama, M.; Luxburg, U.; Guyon, I.; and Garnett, R.,
  eds., \emph{Advances in Neural Information Processing Systems}, volume~29.
  Curran Associates, Inc.

\bibitem[{Hussein et~al.(2017)Hussein, Gaber, Elyan, and Jayne}]{imitation2017}
Hussein, A.; Gaber, M.~M.; Elyan, E.; and Jayne, C. 2017.
\newblock Imitation Learning: A Survey of Learning Methods.
\newblock \emph{ACM Comput. Surv.}, 50(2).

\bibitem[{Jenner, van Hoof, and Gleave(2022)}]{MDPcalculus}
Jenner, E.; van Hoof, H.; and Gleave, A. 2022.
\newblock Calculus on MDPs: Potential Shaping as a Gradient.

\bibitem[{Jeon, Milli, and Dragan(2020)}]{RRIC}
Jeon, H.~J.; Milli, S.; and Dragan, A. 2020.
\newblock Reward-rational (implicit) choice: A unifying formalism for reward
  learning.
\newblock In Larochelle, H.; Ranzato, M.; Hadsell, R.; Balcan, M.; and Lin, H.,
  eds., \emph{Advances in Neural Information Processing Systems}, volume~33,
  4415--4426. Curran Associates, Inc.

\bibitem[{Kahneman and Tversky(1979)}]{prospecttheory}
Kahneman, D.; and Tversky, A. 1979.
\newblock Prospect Theory: An Analysis of Decision under Risk.
\newblock \emph{Econometrica}, 47(2): 263--291.

\bibitem[{Kim et~al.(2021)Kim, Garg, Shiragur, and Ermon}]{kim2021}
Kim, K.; Garg, S.; Shiragur, K.; and Ermon, S. 2021.
\newblock Reward Identification in Inverse Reinforcement Learning.
\newblock In \emph{Proceedings of the 38th International Conference on Machine
  Learning}, volume 139 of \emph{Proceedings of Machine Learning Research},
  5496--5505. Virtual: PMLR.

\bibitem[{Ng, Harada, and Russell(1999)}]{ng1999}
Ng, A.~Y.; Harada, D.; and Russell, S. 1999.
\newblock Policy Invariance Under Reward Transformations: Theory and
  Application to Reward Shaping.
\newblock In \emph{Proceedings of the Sixteenth International Conference on
  Machine Learning}, 278--287. Bled, Slovenia: Morgan Kaufmann Publishers Inc.

\bibitem[{Ng and Russell(2000)}]{ng2000}
Ng, A.~Y.; and Russell, S. 2000.
\newblock Algorithms for Inverse Reinforcement Learning.
\newblock In \emph{Proceedings of the Seventeenth International Conference on
  Machine Learning}, volume~1, 663--670. Stanford, California, USA: Morgan
  Kaufmann Publishers Inc.

\bibitem[{Orsini et~al.(2021)Orsini, Raichuk, Hussenot, Vincent, Dadashi,
  Girgin, Geist, Bachem, Pietquin, and Andrychowicz}]{orsini2021}
Orsini, M.; Raichuk, A.; Hussenot, L.; Vincent, D.; Dadashi, R.; Girgin, S.;
  Geist, M.; Bachem, O.; Pietquin, O.; and Andrychowicz, M. 2021.
\newblock What Matters for Adversarial Imitation Learning?
\newblock \emph{arXiv preprint}, arXiv:2106.00672 [cs.LG].
\newblock To appear in \textit{Proceedings of the 35th International Conference
  on Neural Information Processing Systems}, 2021.

\bibitem[{Pan, Bhatia, and Steinhardt(2022)}]{reward_misspecification}
Pan, A.; Bhatia, K.; and Steinhardt, J. 2022.
\newblock The Effects of Reward Misspecification: Mapping and Mitigating
  Misaligned Models.

\bibitem[{Ramachandran and Amir(2007)}]{ramachandran2007}
Ramachandran, D.; and Amir, E. 2007.
\newblock Bayesian Inverse Reinforcement Learning.
\newblock In \emph{Proceedings of the 20th International Joint Conference on
  Artifical Intelligence}, 2586--2591. Hyderabad, India: Morgan Kaufmann
  Publishers Inc.

\bibitem[{Rothkopf and Dimitrakakis(2011)}]{rothkopf2011}
Rothkopf, C.~A.; and Dimitrakakis, C. 2011.
\newblock Preference Elicitation and Inverse Reinforcement Learning.
\newblock In \emph{Machine Learning and Knowledge Discovery in Databases: ECML
  PKDD 2011, Proceedings, Part III}, volume 6913 of \emph{Lecture Notes in
  Computer Science}, 34--48. Athens, Greece: Springer.

\bibitem[{Shah et~al.(2019)Shah, Gundotra, Abbeel, and
  Dragan}]{shah2019feasibility}
Shah, R.; Gundotra, N.; Abbeel, P.; and Dragan, A.~D. 2019.
\newblock On the Feasibility of Learning, Rather than Assuming, Human Biases
  for Reward Inference.
\newblock arXiv:1906.09624.

\bibitem[{Singh et~al.(2019)Singh, Yang, Hartikainen, Finn, and
  Levine}]{singh2019}
Singh, A.; Yang, L.; Hartikainen, K.; Finn, C.; and Levine, S. 2019.
\newblock End-to-End Robotic Reinforcement Learning Without Reward Engineering.
\newblock In \emph{Proceedings of Robotics: Science and Systems}. Freiburg im
  Breisgau, Germany.

\bibitem[{Skalse et~al.(2022{\natexlab{a}})Skalse, Farrugia-Roberts, Russell,
  Abate, and Gleave}]{skalse2022}
Skalse, J.; Farrugia-Roberts, M.; Russell, S.; Abate, A.; and Gleave, A.
  2022{\natexlab{a}}.
\newblock Invariance in Policy Optimisation and Partial Identifiability in
  Reward Learning.
\newblock \emph{arXiv preprint arXiv:2203.07475}.

\bibitem[{Skalse et~al.(2022{\natexlab{b}})Skalse, Howe, Dima, and
  Krueger}]{rewardhacking}
Skalse, J.; Howe, N.; Dima, K.; and Krueger, D. 2022{\natexlab{b}}.
\newblock Defining and Characterizing Reward Hacking.
\newblock In \emph{Proceedings of the 33rd International Conference on Neural
  Information Processing Systems}.

\bibitem[{Sutton and Barto(2018)}]{sutton2018}
Sutton, R.~S.; and Barto, A.~G. 2018.
\newblock \emph{Reinforcement Learning: An Introduction}.
\newblock MIT Press, second edition.
\newblock ISBN 9780262352703.

\bibitem[{White(1994)}]{white_1994}
White, H. 1994.
\newblock \emph{Estimation, Inference and Specification Analysis}.
\newblock Econometric Society Monographs. Cambridge University Press.

\bibitem[{Yamaguchi et~al.(2018)Yamaguchi, Naoki, Ikeda, Tsukada, Nakano, Mori,
  and Ishii}]{celegansIRL}
Yamaguchi, S.; Naoki, H.; Ikeda, M.; Tsukada, Y.; Nakano, S.; Mori, I.; and
  Ishii, S. 2018.
\newblock Identification of animal behavioral strategies by inverse
  reinforcement learning.
\newblock \emph{PLOS Computational Biology}, 14(5): 1--20.

\bibitem[{Ziebart(2010)}]{ziebart2010thesis}
Ziebart, B.~D. 2010.
\newblock \emph{Modeling Purposeful Adaptive Behavior with the Principle of
  Maximum Causal Entropy}.
\newblock Ph.D. thesis, Carnegie Mellon University.

\end{thebibliography}

\appendix

\newpage
\setcounter{theorem}{0}
\setcounter{section}{0}
\newpage

\section{Proofs}

In this Appendix, we provide the proofs of all our results, as well as of some additional lemmas.

\subsection{Fundamental Lemmas}

We here prove the fundamental lemmas from Section~\ref{section:lemmas}, as well as some additional lemmas. Our proofs in this section are given relative to the somewhat more general definitions of $P$-robustness and refinement given in Section~\ref{section:subspaces}, rather than those given in Section~\ref{section:framweork}.

\subsubsection{Additional Lemmas}

We here provide a few extra lemmas, which are straightforward to prove, but worth spelling out.

\begin{lemma}\label{lemma:no_P_adm_inherited}
For any $f$ and $h$, if $f$ is not $P$-admissible on $\Rspace$ then $h \circ f$ is not $P$-admissible on $\Rspace$.
\end{lemma}
\begin{proof}
If $f$ is not $P$-admissible on $\Rspace$ then there are $R_1,R_2 \in \Rspace$ such that $f(R_1) = f(R_2)$, but $R_1 \not\equiv_P R_2$. But if $f(R_1) = f(R_2)$ then $h \circ f(R_1) = h \circ f(R_2)$, so there are $R_1,R_2 \in \Rspace$ such that $h \circ f(R_1) = h \circ f(R_2)$, but $R_1 \not\equiv_P R_2$. Thus $h \circ f$ is not $P$-admissible on $\Rspace$.
\end{proof}

\begin{lemma}\label{lemma:rob_to_adm}
If $f$ is $P$-robust to misspecification with $g$ on $\Rspace$ then $g$ is $P$-admissible on $\Rspace$.
\end{lemma}
\begin{proof}
Suppose that $f$ is $P$-robust to misspecification with $g$ on $\Rspace$, and let $R_1,R_2 \in \Rspace$ be any two reward functions such that $g(R_1) = g(R_2)$.
Since $\mathrm{Im}(g|_\Rspace) \subseteq \mathrm{Im}(f|_\Rspace)$ there is an $R_3 \in \Rspace$ such that $f(R_3) = g(R_1) = g(R_2)$. 
Since $f$ is $P$-robust to misspecification with $g$ on $\Rspace$, it must be the case that $R_3 \equiv_P R_1$ and $R_3 \equiv_P R_2$. By transitivity, we thus have that $R_1 \equiv_P R_2$. Since $R_1$ and $R_2$ were chosen arbitrarily, it must be that $R_1 \equiv_P R_2$ whenever $g(R_1) = g(R_2)$.
\end{proof}

\begin{lemma}\label{lemma:P_rob_symmetry}
If $f$ is $P$-robust to misspecification with $g$ on $\Rspace$ and $\mathrm{Im}(f|_\Rspace) = \mathrm{Im}(g|_\Rspace)$ then $g$ is $P$-robust to misspecification with $f$ on $\Rspace$.
\end{lemma}
\begin{proof}
If $f$ is $P$-robust to misspecification with $g$ on $\Rspace$ then this immediately implies that $f_\Rspace \neq g|_\Rspace$, and that if $f(R_1) = g(R_2)$ for some $R_1,R_2 \in \Rspace$ then $R_1 \equiv_P R_2$. Lemma~\ref{lemma:rob_to_adm} implies that $g$ is $P$-admissible on $\Rspace$, and if $\mathrm{Im}(f|_\Rspace) = \mathrm{Im}(g|_\Rspace)$ then $\mathrm{Im}(f|_\Rspace) \subseteq \mathrm{Im}(g|_\Rspace)$. This means that $g$ is $P$-robust to misspecification with $f$ on $\Rspace$.
\end{proof}

\begin{lemma}\label{lemma:ambiguity_robustness}
$f$ is $P$-admissible on $\Rspace$ but not $P$-robust to any misspecification on $\Rspace$ if and only if $\mathrm{Am}(f) = P$ on $\Rspace$.
\end{lemma}
\begin{proof}
First suppose $\mathrm{Am}(f) = P$ on $\Rspace$. This immediately implies that $f$ is $P$-admissible on $\Rspace$. Next, assume that $f$ is $P$-robust to misspecification with $g$ on $\Rspace$, let $R_1$ be any element of $\Rspace$, and consider $g(R_1)$. Since $\mathrm{Im}(g|_\Rspace) \subseteq \mathrm{Im}(f|_\Rspace)$, there is an $R_2 \in \Rspace$ such that $f(R_2) = g(R_1)$. Since $f$ is $P$-robust to misspecification with $g$ on $\Rspace$, this implies that $R_2 \equiv_P R_1$. Moreover, if $\mathrm{Am}(f) = P$ then $R_2 \equiv_P R_1$ if and only if $f(R_2) = f(R_1)$, so it must be the case that $f(R_2) = f(R_1)$. Now, since $f(R_2) = f(R_1)$ and $f(R_2) = g(R_1)$, we have that $g(R_1) = f(R_1)$. Since $R_1$ was chosen arbitrarily, this implies that $f|_\Rspace = g|_\Rspace$, which is a contradiction.
Hence, if $\mathrm{Am}(f) = P$ on $\Rspace$ then $f$ is $P$-admissible on $\Rspace$ but not $P$-robust to any misspecification on $\Rspace$.

For the other direction, suppose that $f$ is $P$-admissible on $\Rspace$ and that $\mathrm{Am}(f) \neq P$ on $\Rspace$. If $\mathrm{Am}(f) \neq P$ on $\Rspace$ then there are $R_1, R_2 \in \Rspace$ such that $R_1 \equiv_P R_2$ but $f(R_1) \neq f(R_2)$. We can then construct a $g$ as follows; let $g(R_1) = f(R_2)$, $g(R_2) = f(R_1)$, and $g(R) = f(R)$ for all $R \neq R_1,R_2$. Now $f$ is $P$-robust to misspecification with $g$ on $\Rspace$. Hence, if $f$ is $P$-admissible on $\Rspace$ but not $P$-robust to any misspecification on $\Rspace$ then $\mathrm{Am}(f) = P$ on $\Rspace$.
\end{proof}

\subsubsection{Main Lemmas}

We here prove the lemmas from \ref{section:lemmas}.

\begin{lemma}
If $f$ is not $P$-robust to misspecification with $g$ on $\Rspace$, and $\mathrm{Im}(g|_\Rspace) \subseteq \mathrm{Im}(f|_\Rspace)$, then for any $h$, $h \circ f$ is not $P$-robust to misspecification with $h \circ g$ on $\Rspace$.
\end{lemma}
\begin{proof}
If $f$ is not $P$-robust to misspecification with $g$ on $\Rspace$, and $\mathrm{Im}(g|_\Rspace) \subseteq \mathrm{Im}(f|_\Rspace)$, then either $f$ is not $P$-admissible on $\Rspace$, or $f|_\Rspace = g|_\Rspace$, or $f(R_1) = g(R_2)$ but $R_1 \not\equiv_P R_2$ for some $R_1, R_2 \in \Rspace$.

In the first case, if $f$ is not $P$-admissible on $\Rspace$ then $h \circ f$ is not $P$-admissible on $\Rspace$, as per Lemma~\ref{lemma:no_P_adm_inherited}. This implies that $h \circ f$ is not $P$-robust to any misspecification (including with $h \circ g$) on $\Rspace$.

In the second case, if $f|_\Rspace = g|_\Rspace$ then $h \circ f|_\Rspace = h \circ g|_\Rspace$. This implies that $h \circ f$ is not $P$-robust to misspecification with $h \circ g$ on $\Rspace$.

In the last case, suppose $f(R_1) = g(R_2)$ but $R_1 \not\equiv_P R_2$ for some $R_1, R_2 \in \Rspace$. If $f(R_1) = g(R_2)$ then $h \circ f(R_1) = h \circ g(R_2)$, so there are $R_1, R_2 \in \Rspace$ such that $h \circ f(R_1) = h \circ g(R_2)$, but $R_1 \not\equiv_P R_2$. This implies that $h \circ f$ is not $P$-robust to misspecification with $h \circ g$ on $\Rspace$.
\end{proof}

\begin{lemma}
Let $f$ be $P$-admissible on $\Rspace$, and let $T$ be the set of all reward transformations that preserve $P$ on $\Rspace$. 
Then $f$ is $P$-robust to misspecification with $g$ on $\Rspace$ if and only if $g = f \circ t$ for some $t \in T$ such that $f \circ t|_\Rspace \neq f|_\Rspace$.
\end{lemma}

\begin{proof}
First suppose that $f$ is $P$-robust to misspecification with $g$ on $\Rspace$ --- we will construct a $t$ that fits our description.
For each $y \in \mathrm{Im}(g|_\Rspace)$, let $R_y \in \Rspace$ be some reward function such that $f(R_y) = y$; since $\mathrm{Im}(g|_\Rspace) \subseteq \mathrm{Im}(f|_\Rspace)$, such an $R_y \in \Rspace$ always exists. 
Now let $t$ be the function that maps each $R \in \Rspace$ to $R_{g(R)}$. Since by construction $g(R) = f(R_{g(R)})$, and since $f$ is $P$-robust to misspecification with $g$ on $\Rspace$, we have that $R \equiv_P R_{g(R)}$. This in turn means that $t \in T$, since $t$ preserves $P$ on $\Rspace$. Finally, note that $g = f \circ t$, which means that we are done.

For the other direction, suppose $g = f \circ t$ for some $t \in T$ where $f \circ t|_\Rspace \neq f|_\Rspace$. By assumption we have that $f$ is $P$-admissible on $\Rspace$, and that $g|_\Rspace \neq f|_\Rspace$.
Moreover, we clearly have that $\mathrm{Im}(g|_\Rspace) \subseteq \mathrm{Im}(f|_\Rspace)$. 
Finally, if $g(R_1) = f(R_2)$ then $f \circ t(R_1) = f(R_2)$, which means that $R_1 \equiv_P R_3$ for some $R_3 \in \Rspace$ such that $f(R_3) = f(R_2)$. Since $f$ is $P$-admissible on $\Rspace$ it follows that $R_3 \equiv_P R_2$, which then implies that $R_1 \equiv_P R_2$. Thus $f$ is $P$-robust to misspecification with $g$ on $\Rspace$, so we are done.
\end{proof}

\subsection{Reward Function Equivalence Classes}

In this section we will prove Theorem~\ref{thm:order_of_policies}, which turns out to be quite involved. We start by proving several lemmas, which we will need for the main proof.

\subsubsection{Lemmas Concerning State-Action Visit Counts}

Here we provide some lemmas about the topological structure of MDPs. Recall that we assume that all states in $S$ are reachable under $\TransitionDistribution$ and $\InitStateDistribution$.

Let $\Pi$ be the set of all policies. Moreover, given $\TransitionDistribution$ and $\InitStateDistribution$, let $m_{\TransitionDistribution,\InitStateDistribution} : \Pi \rightarrow \mathbb{R}^{|S||A|}$ be a map that sends each policy $\pi$ to a vector $d_\pi$, such that
$$
d_\pi[s,a] = \sum_{t=0}^\infty \gamma^t \mathbb{P}_{\xi \sim \pi}\left(S_t,A_t = s,a \right). 
$$
In other words, let $m_{\TransitionDistribution,\InitStateDistribution}(\pi)$ be a vector that records the expected discounted \enquote{density} of $\pi$'s trajectories in each state-action pair under $\TransitionDistribution$ and $\InitStateDistribution$. In some sources, $m_{\TransitionDistribution,\InitStateDistribution}(\pi)$ is referred to as the occupancy measure of $\pi$.
Moreover, given a reward function $R$ and a transition function $\TransitionDistribution$, let $\vec{R}^\TransitionDistribution \in \mathbb{R}^{|S||A|}$ be the vector where 
$$
\vec{R}^\TransitionDistribution[s,a] = \mathbb{E}_{S' \sim \TransitionDistribution(s,a)}[R(s,a,S')].
$$ 
We will refer to these vectors as \emph{reward vectors}, and will for the sake of clarity distinguish them from reward functions. 
Note that $\vec{R}_1^\tau = \vec{R}_1^\tau$ if and only if $R_1$ and $R_2$ differ by $\SR$.
Also note that $\Evaluation(\pi) = m_{\TransitionDistribution,\InitStateDistribution}(\pi)\cdot \vec{R}^\TransitionDistribution$.
This means that we can use $m_{\TransitionDistribution,\InitStateDistribution}$ to decompose $\Evaluation$ into two separate steps.

Let $\bar{\Pi} \subset \Pi$ be the set of all policies that visit each state with positive probability. 

\begin{lemma}\label{lemma:m_injective}
$m_{\TransitionDistribution,\InitStateDistribution}$ is injective on $\bar{\Pi}$.
\end{lemma}
\begin{proof}
Suppose $m_{\TransitionDistribution,\InitStateDistribution}(\pi) = m_{\TransitionDistribution,\InitStateDistribution}(\pi')$ for some $\pi,\pi' \in \bar{\Pi}$. Next, given $\TransitionDistribution,\InitStateDistribution$, define $w_\pi$ as
$$
w_\pi(s) = \sum_{t=0}^\infty \gamma^t \mathbb{P}_{\xi \sim \pi}(S_t = s).
$$
Note that if $m_{\TransitionDistribution,\InitStateDistribution}(\pi) = m_{\TransitionDistribution,\InitStateDistribution}(\pi')$ then $w_\pi = w_{\pi'}$, and moreover that
$$
m_{\TransitionDistribution,\InitStateDistribution}(\pi)[s,a] = w_\pi(s)\pi(a \mid s).
$$
This means that if $w_{\pi}(s) \neq 0$ for all $s$, which is the case for all $\pi \in \bar{\Pi}$, then we can express $\pi$ as
$$
\pi(a \mid s) = \frac{m_{\TransitionDistribution,\InitStateDistribution}(\pi)[s,a]}{w_\pi(s)}.
$$
This means that if $m_{\TransitionDistribution,\InitStateDistribution}(\pi) = m_{\TransitionDistribution,\InitStateDistribution}(\pi')$ for some $\pi,\pi' \in \bar{\Pi}$ then $\pi = \pi'$.
\end{proof}

Note that $m_{\TransitionDistribution,\InitStateDistribution}$ is \emph{not} injective on $\Pi$; if there is some state $s$ that $\pi$ reaches with probability $0$, then we can alter the behaviour of $\pi$ at $s$ without changing $m_{\TransitionDistribution,\InitStateDistribution}(\pi)$. 


\begin{lemma}\label{lemma:image_dimension}
$\mathrm{Im}(m_{\TransitionDistribution,\InitStateDistribution})$ is located in an affine space with no more than $|\States|(|\Actions|-1)$ dimensions.
\end{lemma}
\begin{proof}
We wish to establish an \emph{upper} bound on the number of linearly independent vectors in $\mathrm{Im}(m_{\TransitionDistribution,\InitStateDistribution})$.
We can do this by establishing a \emph{lower} bound on the size of the space of all reward functions that share the same policy evaluation function, $\Evaluation$.
To see this, consider the fact that $\Evaluation(\pi) = m_{\TransitionDistribution,\InitStateDistribution}(\pi)\cdot \vec{R}^\TransitionDistribution$. 
We have that $\Vec{R}^\TransitionDistribution$ is an $|\States||\Actions|$-dimensional vector.
Consider a reward vector $\vec{R}^\tau_1$, and let $X$ be the space of all reward vectors $\vec{R}^\tau_2$ such that $\vec{R}^\tau_1 \cdot d = \vec{R}^\tau_2 \cdot d$ for all $d \in \mathrm{Im}(m_{\TransitionDistribution,\InitStateDistribution})$.
It is then a straightforward consequence of linear algebra that if $\mathrm{Im}(m_{\TransitionDistribution,\InitStateDistribution})$ contains $n$ linearly independent vectors, then $X$ forms an affine space with $|\States||\Actions| - n$ dimensions. We can thus obtain an upper bound on the number of linearly independent vectors in $\mathrm{Im}(m_{\TransitionDistribution,\InitStateDistribution})$ from a lower bound on the dimensionality of $X$.

Next, recall that if $R_2$ is produced by potential shaping $R_1$ with $\Phi$, and $\mathbb{E}_{S_0 \sim \InitStateDistribution}\left[\Phi(S_0)\right] = 0$, then $\Evaluation_1(\pi) = \Evaluation_2(\pi)$ for all $\pi$. 
This means that for any $\vec{R}^\tau_1$, we have that $X$ contains all vectors $\vec{R}^\tau_2$ where
$$
\vec{R}^\tau_2[s,a] = \vec{R}^\tau_1[s,a] + \gamma \mathbb{E}_{S' \sim \tau(s,a)}[\Phi(S')] - \Phi(s)
$$
for some potential function $\Phi$ where $\mathbb{E}_{S_0 \sim \InitStateDistribution}\left[\Phi(S_0)\right] = 0$. The space of all such reward vectors is an affine space with $|\States| - 1$ dimensions. This means that $\mathrm{Im}(m_{\TransitionDistribution,\InitStateDistribution})$ contains at most $|\States|(|\Actions|-1) + 1$ linearly independent vectors.

Next, note that there is no $\pi$ such that $m_{\TransitionDistribution,\InitStateDistribution}(\pi)$ is the zero vector. In fact, $\sum m_{\TransitionDistribution,\InitStateDistribution}(\pi) = 1/(1-\gamma)$ for all $\pi$. This means that the smallest affine space which contains $\mathrm{Im}(m_{\TransitionDistribution,\InitStateDistribution})$ does not contain the origin. Therefore, $\mathrm{Im}(m_{\TransitionDistribution,\InitStateDistribution})$ is located in an affine space with no more than $|\States|(|\Actions|-1)$ dimensions.
\end{proof}

For the next lemma, let $\tilde{\Pi} \subset \Pi$ be the set of all policies that take all actions with positive probability in each state, and note that $\tilde{\Pi} \subset \bar{\Pi}$ (i.e., a policy that takes every action with positive probability in each state visits every state with positive probability). 

\begin{lemma}\label{lemma:homeomorphism}
$\mathrm{Im}(m_{\TransitionDistribution,\InitStateDistribution})$ is located in an affine space with $|\States|(|\Actions|-1)$ dimensions, in which $m_{\TransitionDistribution,\InitStateDistribution}(\tilde{\Pi})$ is an open set.
\end{lemma}

\begin{proof}
By the Invariance of Domain theorem, if 
\begin{enumerate}
    \item $U$ is an open subset of $\mathbb{R}^n$, and
    \item $f : U \rightarrow \mathbb{R}^n$ is an injective continuous map,
\end{enumerate}
then $f(U)$ is open in $\mathbb{R}^n$ (and $f$ is a homeomorphism between $U$ and $f(U)$). We will show that $m$ and $\tilde{\Pi}$ satisfy the requirements of this theorem.

We begin by noting that $\Pi$ can be represented as a set of points in $\mathbb{R}^{|\States|(|\Actions|-1)}$.
We do this by considering 
each policy $\pi$ as a vector $\vec{\pi}$ of length $|\States||\Actions|$, where $\vec{\pi}[s,a] = \pi(a \mid s)$. 
Moreover, since $\sum_{a \in A} \pi(a \mid s) = 1$ for all $s$, 
we can remove $|\Actions|$ dimensions, and embed $\Pi$ in $\mathbb{R}^{|\States|(|\Actions|-1)}$.

$\tilde{\Pi}$ is an open set in $\mathbb{R}^{|\States|(|\Actions|-1)}$. By Lemma~\ref{lemma:image_dimension}, we have that $m_{\TransitionDistribution,\InitStateDistribution}$ is a mapping 
from $\tilde{\Pi}$ to an affine space with no more than $|\States|(|\Actions|-1)$ dimensions.
By Lemma~\ref{lemma:m_injective}, we have that $m_{\TransitionDistribution,\InitStateDistribution}$ is injective on $\tilde{\Pi}$. Finally, $m_{\TransitionDistribution,\InitStateDistribution}$ is continuous (it can be expressed as a uniformly convergent series of continuous functions). We can therefore apply the Invariance of Domain theorem, and conclude that 
$\mathrm{Im}(m_{\TransitionDistribution,\InitStateDistribution})$ is located in an affine space with $|\States|(|\Actions|-1)$ dimensions, in which $m_{\TransitionDistribution,\InitStateDistribution}(\tilde{\Pi})$ is an open set.
\end{proof}

Note that lemma~\ref{lemma:homeomorphism} holds for all $\TransitionDistribution$ and $\InitStateDistribution$ (for which all states are reachable).

\subsubsection{Results Concerning the Policy Order}

In this section, we prove our results concerning the policy orderings.
First, we need to define a new set of transformations. Let $\kPS{k}$ be the set of all potential shaping transformations $t$ that, for each $R$, apply a potential function $\Phi$ such that $\mathbb{E}_{S_0 \sim \InitStateDistribution}[\Phi(S_0)] = k$.

\begin{lemma}\label{lemma:ambiguity_of_J}
$\Evaluation_1 = \Evaluation_2$ if and only if $R_1 = t(R_2)$ for some $t \in \kPS{0} \circ \SR$.
\end{lemma}
\begin{proof}
For the first direction, suppose $R_1 = t(R_2)$ for some $t \in \kPS{0} \circ \SR$. Then $V_1^\pi(s) = V_2^\pi(s) - \Phi(s)$, where $\Phi$ is the potential shaping function applied by $t$ \cite[see e.g. Lemma~B1 in ][]{skalse2022}. Hence $\Evaluation_1(\pi) = \Evaluation_2(\pi) - \mathbb{E}_{s_0 \sim \InitStateDistribution}[\Phi(s_0)] = \Evaluation_2(\pi)$, and so we have proven the first direction.

For the other direction, first recall that $\Evaluation(\pi) = m_{\TransitionDistribution,\InitStateDistribution}(\pi) \cdot \Vec{R}^\TransitionDistribution$.
Next, Lemma~\ref{lemma:homeomorphism} implies that $\mathrm{Im}(m_{\TransitionDistribution,\InitStateDistribution})$ contains $|S|(|A|-1) + 1$ linearly independent vectors. 
It is then a straightforward fact of linear algebra that, for any reward vector $\vec{R}_1^\tau$, the space $X$ of all reward vectors $\vec{R}_2^\tau$ such that $\vec{R}^\tau_1 \cdot d = \vec{R}^\tau_2 \cdot d$ for all $d \in \mathrm{Im}(m_{\TransitionDistribution,\InitStateDistribution})$, forms an affine space with $|\States|-1$ dimensions.

We know that $\Evaluation$ is preserved by transformations in $\kPS{0} \circ \SR$. Next, given $R_1$, the space of all reward vectors $\vec{R}_2^\tau$ given by some $R_2$ such that $R_2 = t(R_1)$ for some $t \in \kPS{0} \circ \SR$, forms an affine space with $|\States|-1$ dimensions. Since this space is contained in $X$, and since they have the same number of dimensions, they must be one and the same.

Therefore, given $R_1$, if $\Evaluation_1 = \Evaluation_2$, then there exists an $R_3$ such that $R_3 = t(R_1)$ for some $t \in \kPS{0} \circ \SR$, and such that $\vec{R}_3^\tau = \vec{R}_2^\tau$. But this means that $R_2 = t(R_1)$ for some $t \in \kPS{0} \circ \SR$. We have thus proven the other direction, which completes the proof.
\end{proof}


We can now finally prove Theorem~\ref{thm:order_of_policies}.

\begin{theorem}
$R_1 \equiv_\mathrm{ORD^\M} R_2$ if and only if $R_2 = t(R_1)$ for some $t \in \SR \circ \PS \circ \LS$.
\end{theorem}
\begin{proof}
First, $R_1 \equiv_\mathrm{ORD^\M} R_2$ if and only if $\Evaluation_1$ is a monotonic transformation of $\Evaluation_2$. 
Next, since $\Evaluation(\pi) = m_{\TransitionDistribution,\InitStateDistribution}(\pi) \cdot \Vec{R}^\TransitionDistribution$, we have that 
all possible monotonic transformations of $\Evaluation$ are affine.
Hence $R_1 \equiv_\mathrm{ORD^\M} R_2$ if and only if $\Evaluation_1 = a \cdot \Evaluation_2 + b$ for some $a \in \mathbb{R}^+, b \in \mathbb{R}$.

The first direction is straightforward. First, if $R_1 = t(R_2)$ for some $t \in \SR$ then $\Evaluation_1 = \Evaluation_2$. Next, if $R_1 = t(R_2)$ for some $t \in \PS$ then $\Evaluation_1 = \Evaluation_2 - \mathbb{E}_{S_0 \sim \InitStateDistribution}[\Phi_t(S_0)]$ \cite[see e.g. Lemma~B1 in ][]{skalse2022}. Finally, if $R_1 = t(R_2)$ for some $t \in \LS$ then $\Evaluation_1 = c \cdot \Evaluation_2$ for some $c \in \mathbb{R}^+$. Hence if $R_1 = t(R_2)$ for some $t \in \SR \circ \PS \circ \LS$ then $\Evaluation_1 = a \cdot \Evaluation_2 + b$ for some $a \in \mathbb{R}^+, b \in \mathbb{R}$.

For the other direction, suppose $\Evaluation_1 = a \cdot \Evaluation_2 + b$ for some $a \in \mathbb{R}^+, b \in \mathbb{R}$. Consider the reward function $R_3$ given by first scaling $R_2$ by $a$, and then shape the resulting reward with the potential function $\Phi$ that is equal to $-b$ for all initial states, and equal to $0$ elsewhere. Now $\Evaluation_3 = \Evaluation_1$, so (by Lemma~\ref{lemma:ambiguity_of_J}) there is a $t' \in \kPS{0} \circ \SR$ such that $R_1 = t'(R_3)$. By composing $t'$ with the transformation that produced $R_3$ from $R_2$, we obtain a $t \in \SR \circ \PS \circ \LS$ such that $R_1 = t(R_2)$. Hence if $R_1 \equiv_\mathrm{ORD^\M} R_2$ then $R_1 = t(R_2)$ for some $t \in \SR \circ \PS \circ \LS$. We have thus proven both directions.
\end{proof} 



\subsection{Misspecified Behavioural Models}

In this section, we prove our results from Section~\ref{section:misspecified_policies}.

\begin{theorem}
Let $f^\M \in F^\M$ be surjective onto $\Pi^+$. Then $f^M$ is $\mathrm{OPT^\M}$-robust to misspecification with $g$ if and only if $g \in F^\M$ and $g \neq f^\M$.
\end{theorem}
\begin{proof}
$f^\M$ is $\mathrm{OPT^\M}$-robust to misspecification with $g$ in $\M$ if and only if $f^\M$ is $\mathrm{OPT^\M}$-admissible, $g \neq f^\M$, $\mathrm{Im}(g) \subseteq \mathrm{Im}(f)$, and if $f^\M(R_1) = g(R_2)$ then $R_1$ and $R_2$ have the same optimal policies in $\M$.

For all $f \in F^\M$ and all $R$,
$$
\mathrm{argmax}_{a \in \Actions} f(R)(a \mid s) = \mathrm{argmax}_{a \in \Actions} \QStar(s,a).
$$
Since $f^\M \in F^\M$, this means that if $f^\M(R_1) = f^\M(R_2)$ then $\mathrm{argmax}_{a \in \Actions} Q^\star_1(s,a) = \mathrm{argmax}_{a \in \Actions} Q^\star_2(s,a)$ in $\M$. 
Moreover, $R_1$ and $R_2$ have the same optimal policies in $\M$ if and only if $\mathrm{argmax}_{a \in \Actions} Q^\star_1(s,a) = \mathrm{argmax}_{a \in \Actions} Q^\star_2(s,a)$ in $\M$. 
Thus, if $f^\M(R_1) = f^\M(R_2)$ then $R_1 \equiv_{\mathrm{OPT^\M}} R_2$, and so $f^\M$ is $\mathrm{OPT^\M}$-admissible.

Let $g \in F^\M$ and $g \neq f^\M$. Since $g$ is a function $\mathcal{R} \to \Pi^+$, and since $f^\M$ is surjective onto $\Pi^+$, we have that  $\mathrm{Im}(g) \subseteq \mathrm{Im}(f)$. Next, by the same argument as above, if $f^\M(R_1) = g(R_2)$ then $\mathrm{argmax}_{a \in \Actions} Q^\star_1(s,a) = \mathrm{argmax}_{a \in \Actions} Q^\star_2(s,a)$, which implies that $R_1 \equiv_{\mathrm{OPT^\M}} R_2$. This means that $f^M$ is $\mathrm{OPT^\M}$-robust to misspecification with $g$.

Next, suppose $f^M$ is $\mathrm{OPT^\M}$-robust to misspecification with $g$. This means that $\mathrm{Im}(g) \subseteq \mathrm{Im}(f)$ and that if $f^\M(R_1) = g(R_2)$ then $\mathrm{argmax}_{a \in \Actions} Q^\star_1(s,a) = \mathrm{argmax}_{a \in \Actions} Q^\star_2(s,a)$. Since $\mathrm{Im}(g) \subseteq \mathrm{Im}(f)$ implies that $g$ is a function $\mathcal{R} \to \Pi^+$, and since $f^\M(R_1) = g(R_2)$ implies that $\mathrm{argmax}_{a \in \Actions} f^\M(R)(a \mid s) = \mathrm{argmax}_{a \in \Actions} g(R)(a \mid s)$, this implies that $g \in F^\M$.
\end{proof}

\begin{theorem}
Let $b_\psi^\M \in B^\M$. Then $b_\psi^\M$ is $\mathrm{ORD^\M}$-robust to misspecification with $g$ if and only if $g \in B^\M$ and $g \neq b_\psi^\M$.
\end{theorem}


\begin{proof}
As per Theorem 3.3 in \citet{skalse2022}, $\mathrm{Am}(b_\psi^\M)$ is characterised by $\mathrm{PS}_\discount \circ S'\mathrm{R}_\TransitionDistribution$, and as per Theorem~\ref{thm:order_of_policies}, $\mathrm{ORD}_\M$ is characterised by $\mathrm{PS}_\discount \circ \mathrm{LS} \circ S'\mathrm{R}_\TransitionDistribution$. Hence $b_\psi^\M$ is $\mathrm{ORD}_\M$-admissible, which means that Lemma~\ref{lemma:how_to_calculate_robustness_sets} implies that $b_\psi^\M$ is $\mathrm{ORD}_\M$-robust to misspecification with $g$
if and only if $g \neq b_\psi^\M$, and there exists a $t \in \mathrm{PS}_\discount \circ \mathrm{LS} \circ S'\mathrm{R}_\TransitionDistribution$ such that $g = b_\psi^\M \circ t$.
Recall that $b_\psi^\M(R)$ is given by
\begin{equation*}
    b_\psi^\M(R)(a \mid s) = \frac{\exp \psi(R) A_R(s,a)}{\sum_{a \in \Actions} \exp \psi(R) A_R(s,a)}.
\end{equation*}
where $A_R$ is the optimal advantage function of $R$ in $\M$. If $g(R) = b_\psi^\M \circ t(R)$ for some $t \in \mathrm{PS}_\discount \circ \mathrm{LS} \circ S'\mathrm{R}_\TransitionDistribution$, then we have that
\begin{align*}
    g(R)(a \mid s) &= \frac{\exp \psi(t(R)) A_{t(R)}(s,a)}{\sum_{a \in \Actions} \exp \psi(t(R)) A_{t(R)}(s,a)} \\
    &= \frac{\exp \psi(t(R)) c_R A_R(s,a)}{\sum_{a \in \Actions} \exp \psi(t(R)) c_R A_R(s,a)},
\end{align*}
where $c_R$ is the linear scaling factor that $t$ applies to $R$. 
Note that the advantage function $\A$ is preserved by both potential shaping and $S'$-redistribution.
Now let $\psi'(R) = \psi(t(R)) \cdot c_R$, and we can see that $g = b_{\psi'}^\M \in B^\M$. We have hence shown that $b_\psi^\M$ is strongly robust to misspecification with $g$ in $\M$ if and only if $g \in B^\M$ and $g \neq b_\psi^\M$.
\end{proof}


\begin{theorem}
No function in $\mathcal{O}^\M$ is $\mathrm{ORD^\M}$-admissible. 
The only function in $\mathcal{O}^\M$ that is $\mathrm{OPT^\M}$-admissible is $o_m^\M$, and this function is not $\mathrm{OPT^\M}$-robust to any misspecification.
\end{theorem}

\begin{proof}
This Theorem largely follows from Lemma~\ref{lemma:ambiguity_robustness}.
First, if $o^\M \in \mathcal{O}^\M$ then $o^\M$ has a finite codomain, whereas there is an uncountable number of $\mathrm{ORD}^\M$-equivalence classes.
This means that $o^\M$ cannot be $\mathrm{ORD^\M}$-admissible.
Moreover, $\mathrm{Am}(o^\M_m) = \mathrm{OPT}^\M$. Therefore, by Lemma~\ref{lemma:ambiguity_robustness}, $o^\M_m$ is $\mathrm{OPT^\M}$-admissible, but not $\mathrm{OPT^\M}$-robust to any misspecification.
Finally, if $o^\M \in \mathcal{O}^\M$ but $o^\M \neq o_m^\M$, then there is a pigeonhole argument to show that there must be at least two $R_1, R_2$ such that $o^\M(R_1) = o^\M(R_2)$ but $R_1 \not\equiv_\mathrm{OPT^\M} R_2$. This means that $o^\M$ is not $\mathrm{OPT^\M}$-admissible.


The pigeonhole argument goes like this: the codomain of each $o^\M \in O^\M$ has $(2^{|\Actions|}-1)^{|\States|}$ elements, and there are $(2^{|\Actions|}-1)^{|\States|}$ $\mathrm{OPT}^\M$-equivalence classes.
This means that if $o^\M$ is $\mathrm{OPT}^\M$-admissible, then there must be a one-to-one correspondence between $\mathrm{OPT}^\M$-equivalence classes and elements of $o^\M$'s codomain, so that there for each equivalence class $C$ is a $y_C$ such that $o^\M(R) = y_C$ if and only if $R \in C$.
Further, say that if $f, g : X \rightarrow \mathcal{P}(Y)$ are set-valued functions, then $f \subseteq g$ if $f(x) \subseteq g(x)$ for all $x \in X$, and $f \subset g$ if $f \subseteq g$ but $g \not\subseteq f$.
Then if $o^\M \in \mathcal{O}^\M$ we have that $o^\M(R) \subseteq o_m^\M(R)$ for all $R$ --- a policy is optimal if and only if it takes only optimal actions, but it need not take all optimal actions.
Moreover, if $o^\M \neq o_m^\M$ then there is an $R_1$ such that $o^\M(R_1) \subset o_m^\M(R_1)$.
Let $R_2$ be a reward function so that $o_m^\M(R_2) = o^\M(R_1)$ --- for any function $\States \rightarrow \mathcal{P}(\Actions) - \varnothing$, there is a reward function for which those are the optimal actions, so there is always some $R_2$ such that $o_m^\M(R_2) = o^\M(R_1)$.
Now either $o^\M(R_2) = o^\M(R_1)$ or $o^\M(R_2) \subset o^\M(R_1)$, since all actions that are optimal under $R_2$ are optimal under $R_1$. In the first case, since $o^\M(R_1) = o^\M(R_2)$ but $R_1 \not\equiv_{\mathrm{OPT}^\M} R_2$, we have that $o^\M$ is not $\mathrm{OPT}^\M$-admissible. In the second case, let $R_3$ be a reward function so that $o_m^\M(R_3) = o^\M(R_2)$, and repeat the same argument. Since there can only be a finite sequence $o^\M(R_n) \subset \dots \subset o^\M(R_2) \subset o^\M(R_1)$, we have that we must eventually find two $R_n, R_{n-1}$ such that $o^\M(R_n) = o^\M(R_{n-1})$ but $R_n \not\equiv_{\mathrm{OPT}^\M} R_{n-1}$.
This means that $o^\M$ cannot be $\mathrm{OPT}^\M$-admissible.
\end{proof}

\begin{theorem}
Let $c_\psi^\M \in C^\M$. Then $c_\psi^\M$ is $\mathrm{ORD^\M}$-robust to misspecification with $g$ if and only if $ g \in C^\M$ and $g \neq c_\psi^\M$.
\end{theorem}
\begin{proof}
As per Theorem 3.4 in \citet{skalse2022}, $\mathrm{Am}(c_\psi^\M)$ is characterised by $\mathrm{PS}_\discount \circ S'\mathrm{R}_\TransitionDistribution$, and as per Theorem~\ref{thm:order_of_policies}, $\mathrm{ORD}_\M$ is characterised by $\mathrm{PS}_\discount \circ \mathrm{LS} \circ S'\mathrm{R}_\TransitionDistribution$. Hence $c_\psi^\M$ is $\mathrm{ORD}_\M$-admissible, which means that Lemma~\ref{lemma:how_to_calculate_robustness_sets} implies that $c_\psi^\M$ is $\mathrm{ORD}_\M$-robust to misspecification with $g$
if and only if $g \neq c_\psi^\M$, and there exists a $t \in \mathrm{PS}_\discount \circ \mathrm{LS} \circ S'\mathrm{R}_\TransitionDistribution$ such that $g = c_\psi^\M \circ t$.
Moreover, $c_\psi^\M(R)$ is the unique policy that maximises the maximal causal entropy objective;
\begin{equation*}
    \Evaluation^\mathrm{MCE}_{\psi(R)}(\pi) = \Evaluation_R(\pi) - \psi(R) \mathbb{E}_{S_t \sim \pi, \TransitionDistribution,\InitStateDistribution}[\discount^t \mathcal{H}(\pi(S_t))].
\end{equation*}
Therefore, if $g(R) = c_\psi^\M \circ t(R)$ then $g(R)$ is the policy 
\begin{align*}
    &\max_\pi\Evaluation^\mathrm{MCE}_{\psi(t(R))}(\pi)\\
    = &\max_\pi\Evaluation_{t(R)}(\pi) - \psi(t(R)) \mathbb{E}_{S_t \sim \pi, \TransitionDistribution,\InitStateDistribution}[\discount^t \mathcal{H}(\pi(S_t))] \\
    = &\max_\pi c_R \cdot \Evaluation_{R}(\pi) - \psi(t(R)) \mathbb{E}_{S_t \sim \pi, \TransitionDistribution,\InitStateDistribution}[\discount^t \mathcal{H}(\pi(S_t))]
\end{align*}
where $c_R$ is the linear scaling factor that $t$ applies to $R$. 
Note that $\Evaluation_R$ is preserved by $S'$-redistribution, and potential shaping can only change $\Evaluation_R$ by inducing a uniform constant shift of $\Evaluation_R$ for all policies. This means that linear scaling is the only transformation in $\mathrm{PS}_\discount \circ \mathrm{LS} \circ S'\mathrm{R}_\TransitionDistribution$ that could affect the maximal causal entropy objective.
Finally, let $\psi'$ be the function $\psi'(R) = \psi(t(R)) \cdot c_R$, and we can see that $g = c_{\psi'}^\M \in C^\M$. We have hence shown that $c_\psi^\M$ is $\mathrm{ORD}^\M$-robust to misspecification with $g$ in $\M$ if and only if $g \in C^\M$ and $g \neq c_\psi^\M$.
\end{proof}

\subsection{Misspecified MDPs}

We here prove our results from Section~\ref{section:misspecified_mdps}. The first of these proofs is straightforward.

\begin{lemma}
If $f^{\TransitionDistribution_1} = f^{\TransitionDistribution_1} \circ t$ for all $t \in S'\mathrm{R}_{\TransitionDistribution_1}$ then $f^{\TransitionDistribution_1}$ is not $\mathrm{OPT}^\M$-admissible for $\M = \langle \States,\Actions,\TransitionDistribution_2, \InitStateDistribution, \_, \gamma \rangle$ unless $\TransitionDistribution_1 = \TransitionDistribution_2$.
\end{lemma}
\begin{proof}
This follows directly from Theorem 4.2 in \citet{skalse2022}.
\end{proof}

To prove the next result, we first need a supporting lemma. 
We say that a state $s$ is \emph{controllable} relative to a transition function $\TransitionDistribution$, initial state distribution $\InitStateDistribution$, and discount $\gamma$, if there exist two policies $\pi$, $\pi'$ such that 
$$
\sum_{t=1}^\infty \gamma^t \mathbb{P}_{\xi \sim \pi}(s_t = s)
\neq
\sum_{t=1}^\infty \gamma^t \mathbb{P}_{\xi \sim \pi'}(s_t = s).
$$
Note that the sum starts from $t=1$. It can therefore be viewed as summing the discounted probability that $\pi$ and $\pi'$ enter $s$ at each time step.
Recall also that $\TransitionDistribution$ is trivial if for all $s \in \States$ and $a,a' \in \Actions$, we have $\TransitionDistribution(s,a) = \TransitionDistribution(s,a')$. 

\begin{lemma}\label{lemma:tau_to_control}
For any $\InitStateDistribution$, $\discount$, and $\TransitionDistribution$, there exists a controllable state if and only if $\TransitionDistribution$ is non-trivial.
\end{lemma}
\begin{proof}
It is straightforward to see that if $\TransitionDistribution$ is trivial then there are no controllable states. 

For the other direction, suppose there are no controllable states. This in turn implies that every policy is optimal under any reward function defined over the domain $\States$. Formally, if $R$ is a reward function such that for each $s \in \States$, we have that $R(s,a_1,s_1) = R(s,a_2,s_2)$ for all $s_1,s_2 \in \States, a_1, a_2 \in \Actions$, and if there are no controllable states, then every policy is optimal under $R$. In particular, every deterministic policy is optimal under all such reward functions.

Given a reward function defined over the domain $\States$, let $\vec{R} \in \mathbb{R}^{|S|}$ be the vector such that $\vec{R}[s]$ is the reward that $R$ assigns to transitions leaving $s$.
Moreover, given a deterministic policy $\pi$, let $T^\pi$ be the $|\States|\times|\States|$-dimensional transition matrix that describes the transitions of $\pi$ under $\TransitionDistribution$.
Then if all deterministic policies are optimal under $R$, we can apply Theorem 3 from \citet{ng2000} and conclude that
$$
(T^{\pi} - T^{\pi'})(I - \discount T^{\pi})^{-1}\vec{R} = 0
$$
for all deterministic policies $\pi$, $\pi'$. If this holds for all $\vec{R}$, we then have that $(T^{\pi} - T^{\pi'})(I - \discount T^{\pi})^{-1}$ is the zero matrix for all deterministic policies $\pi$, $\pi'$. Moreover, since $(I - \discount T^{\pi})^{-1}$ has no zero eigenvalues, this then means that $(T^{\pi} - T^{\pi'})$ must be the zero matrix for all pairs of deterministic policies $\pi$, $\pi'$. This, in turn, implies that $\TransitionDistribution$ must be trivial.
\end{proof}

We can now prove the second lemma from Section~\ref{section:misspecified_mdps}.

\begin{lemma}
If $f^{\gamma_1} = f^{\gamma_1} \circ t$ for all $t \in \mathrm{PS}_{\gamma_1}$ then $f^{\gamma_1}$ is not $\mathrm{OPT}^\M$-admissible for $\M = \langle \States,\Actions,\TransitionDistribution, \InitStateDistribution, \_, \gamma_2 \rangle$ unless $\gamma_1 = \gamma_2$ or $\TransitionDistribution$ is trivial.
\end{lemma}
\begin{proof}
As per Lemma~\ref{lemma:tau_to_control}, if $\TransitionDistribution$ is non-trivial then there is a state $s$ that is controllable relative to $\TransitionDistribution$, $\InitStateDistribution$, and $\gamma_2$.
Let $R_1$ be any reward function, and
let $R_2$ be the reward that is obtained by potential shaping $R_1$ with the discount $\gamma_1$ and the potential function that is equal to $X$ on $s$ (where $X \neq 0$), and $0$ on all other states. Note that there is a $t \in \mathrm{PS}_{\gamma_1}$ such that $R_2 = t(R_1)$, which means that $f^{\gamma_1}(R_1) = f^{\gamma_1}(R_2)$. Next, let $\Delta^\pi = \Evaluation_2(\pi) - \Evaluation_1(\pi)$, evaluated in $\M$.
Moreover, given a policy $\pi$, let
\begin{align*}
    n_2^\pi &= \sum_{t=0}^\infty \gamma_2^t \mathbb{P}(\pi\text{ enters \textit{s} at time \textit{t}}),\\
    x_2^\pi &= \sum_{t=0}^\infty \gamma_2^t \mathbb{P}(\pi\text{ exits \textit{s} at time \textit{t}}).  
\end{align*}
We then have that
$\Delta^\pi = X \cdot (\gamma_1 n^\pi_2 - x^\pi_2)$. We will use $p$ to denote $\InitStateDistribution(s)$. If $\gamma_1 = \gamma_2$ then we know that 
$
\Delta^\pi = -X \cdot p
$
, which gives that
\begin{align*}
    X \cdot (\gamma_2 n^\pi_2 - x^\pi_2) &= -X \cdot p\\
    \gamma_2 n^\pi_2 - x^\pi_2 &= - p\\
    x^\pi_2 &= \gamma_2 n^\pi_2 + p
\end{align*}
By plugging this into the above, and rearranging, we obtain 
$$
\Delta^\pi = X n_2^\pi(\gamma_1 - \gamma_2) + pX.
$$
Moreover, if $s$ is controllable then there are $\pi, \pi'$ such that $n_2^\pi \neq n_2^{\pi'}$, which means that $\Delta^\pi \neq \Delta^{\pi'}$. In particular, there are $\pi, \pi'$ such that $\Delta^\pi \neq \Delta^{\pi'}$, and $\pi$ is optimal under $R_1$, but $\pi'$ is not. Now, if $\gamma_1 \neq \gamma_2$ then by making $X$ sufficiently large or sufficiently small, we can make it so that $\pi'$ is optimal under $R_2$, but $\pi$ is not. Hence $\langle \States, \Actions, \TransitionDistribution,\InitStateDistribution,R_1,\gamma_2\rangle$ and $\langle \States, \Actions, \TransitionDistribution,\InitStateDistribution,R_2,\gamma_2\rangle$ have different optimal policies. 
This means that there are two reward functions $R_1$, $R_2$, such that $f^{\gamma_1}(R_1) = f^{\gamma_1}(R_2)$, but $R_1 \not\equiv_{\mathrm{OPT}^\M} R_2$. Therefore, if $\gamma_1 \neq \gamma_2$ and $\TransitionDistribution$ is non-trivial then $f^{\gamma_1}$ is not $\mathrm{OPT}^\M$-admissible.
\end{proof}

It is worth noting that the above lemma works even if $f^{\gamma_1}$ is only invariant to $\gamma_1$-based potential shaping whose potential is $0$ for all initial states, provided that $\tau$ gives control over some non-initial state. This can be used to generalise the lemma somewhat, since there are some reward objects which are invariant only to such potential shaping; see \citet{skalse2022}.

Using these two lemmas, we can now prove the theorem:

\begin{theorem}
If $f^{\TransitionDistribution_1} = f^{\TransitionDistribution_1} \circ t$ for all $t \in S'\mathrm{R}_{\TransitionDistribution_1}$ and $f^{\TransitionDistribution_2} = f^{\TransitionDistribution_2} \circ t$ for all $t \in S'\mathrm{R}_{\TransitionDistribution_2}$, then $f^{\TransitionDistribution_1}$ is not $\mathrm{OPT}^\M$-robust to misspecification with $f^{\TransitionDistribution_2}$ for any $\M$.
Moreover, if $f^{\gamma_1} = f^{\gamma_1} \circ t$ for all $t \in \mathrm{PS}_{\gamma_1}$ and $f^{\gamma_2} = f^{\gamma_2} \circ t$ for all $t \in \mathrm{PS}_{\gamma_2}$, then $f^{\gamma_1}$ is not $\mathrm{OPT}^\M$-robust to misspecification with $f^{\gamma_2}$ for any $\M$ whose transition function $\TransitionDistribution$ is non-trivial.
\end{theorem}
\begin{proof}
Let $\M = \langle \States,\Actions,\TransitionDistribution, \InitStateDistribution, \_, \gamma \rangle$.
If $f$ is $\mathrm{OPT}^{\M}$-robust to misspecification with $g$ then $f$ must by definition be $\mathrm{OPT}^{\M}$-admissible.
Moreover, Lemma~\ref{lemma:rob_to_adm} says that $g$ must be $\mathrm{OPT}^{\M}$-admissible as well.
This proof will proceed by showing that in each case, at least one of the relevant reward objects fails to be $\mathrm{OPT}^{\M}$-admissible.

If $f^{\TransitionDistribution_1} = f^{\TransitionDistribution_1} \circ t$ for all $t \in S'\mathrm{R}_{\TransitionDistribution_1}$ then Lemma~\ref{lemma:no_t_robustness} says that $f^{\TransitionDistribution_1}$ is not $\mathrm{OPT}^\M$-admissible unless $\TransitionDistribution_1 = \TransitionDistribution$, and similarly for $f^{\TransitionDistribution_2}$. 
If $\TransitionDistribution_1 \neq \TransitionDistribution_2$ then either $\TransitionDistribution_1 \neq \TransitionDistribution$ or $\TransitionDistribution_2 \neq \TransitionDistribution$. 
Hence either $f^{\TransitionDistribution_1}$ or $f^{\TransitionDistribution_2}$ is not $\mathrm{OPT}^\M$-admissible, which means that $f^{\TransitionDistribution_1}$ is not $\mathrm{OPT}^\M$-robust to misspecification with $f^{\TransitionDistribution_2}$.

Similarly, if $f^{\gamma_1} = f^{\gamma_1} \circ t$ for all $t \in \mathrm{PS}_{\gamma_1}$ then
Lemma~\ref{lemma:no_g_robustness} says that $f^{\gamma_1}$ is not $\mathrm{OPT}^\M$-admissible unless $\gamma_1 = \gamma$ or $\TransitionDistribution$ is trivial, and similarly for $f^{\gamma_2}$. 
If $\gamma_1 \neq \gamma_2$ then either $\gamma_1 \neq \gamma$ or $\gamma_2 \neq \gamma$. 
Hence either $f^{\gamma_1}$ or $f^{\gamma_2}$ is not $\mathrm{OPT}^\M$-admissible, unless $\TransitionDistribution$ is trivial, which means that $f^{\gamma_1}$ is not $\mathrm{OPT}^\M$-robust to misspecification with $f^{\gamma_2}$, unless $\TransitionDistribution$ is trivial.
\end{proof}



\subsection{Restrictions on the Reward Function}

Here, we prove our results from Section~\ref{section:subspaces}.

\begin{theorem}
If $f$ is $P$-robust to misspecification with $g$ on $\Rspace$ then $f$ is $P$-robust to misspecification with $g'$ on $\mathcal{R}$ for some $g'$ where $g'|_\Rspace = g|_\Rspace$, unless $f$ is not $P$-admissible on $\mathcal{R}$, and if $f$ is $P$-robust to misspecification with $g$ on $\mathcal{R}$ then $f$ is $P$-robust to misspecification with $g$ on $\Rspace$, unless $f|_\Rspace = g|_\Rspace$.
\end{theorem}

\begin{proof}
Suppose $f$ is $P$-robust to misspecification with $g$ on $\Rspace$, and that $f$ is $P$-admissible on $\mathcal{R}$. We construct a $g'$ as follows; let $g'(R) = g(R)$ for all $R \in \Rspace$, and let $g'(R) = f(R)$ for all $R \not\in \Rspace$. 
Now $f$ is $P$-robust to misspecification with $g'$ on $\mathcal{R}$, and $g(R) = g'(R)$ for all $R \in \Rspace$. The other direction is straightforward.
\end{proof}

\end{document}